\documentclass[11pt]{article}

\usepackage{FRdina4}
\usepackage{amssymb}  
\usepackage{MyMacros}
\usepackage{color}
\usepackage{url}

\usepackage{bbm}
\usepackage{bm}
\usepackage{natbib}
\usepackage{comment}
\usepackage{epsfig}
\usepackage{enumerate}
\usepackage{enumitem}
\usepackage{environ}
\usepackage{etex,etoolbox}
\usepackage[mathscr]{euscript}
\usepackage{mathtools}
\usepackage{authblk}

\usepackage{soul,color}
\soulregister\cite7
\soulregister\ref7
\soulregister\pageref7

\DeclareMathOperator{\posi}{posi}
\newcommand{\gambles}{\mathcal{L}}
\newcommand{\pr}{P}
\newcommand{\lpr}{{\underline{\pr}}}
\newcommand{\upr}{{\overline{\pr}}}

\newcommand{\edesirs}{\mathcal{E}}

\begin{document}
\pagestyle{headings}

\title{\textbf{Information algebras in the theory of imprecise probabilities}}

\author[1]{Arianna Casanova}
\author[2]{Juerg Kohlas}
\author[1]{Marco Zaffalon}

\affil[1]{Istituto Dalle Molle di Studi sull’Intelligenza Artificiale (IDSIA), \protect\\ 
Lugano (Switzerland) \authorcr
  \{\tt arianna, zaffalon\}@idsia.ch}
\affil[2]{Department of Informatics DIUF, University of Fribourg, \protect\\
Fribourg (Switzerland)  \authorcr
 juerg.kohlas@unifr.ch}


\date{\today}

\maketitle


\begin{abstract}
In this paper, we show that coherent sets of gambles and coherent lower and upper previsions can be embedded into the algebraic structure of information algebra. This leads firstly, to a new perspective of the algebraic and logical structure of desirability and imprecise probabilities and secondly, it connects imprecise probabilities to other formalism in computer science sharing the same underlying structure.
Both the domain free and the labeled view of the resulting information algebras are presented, considering product possibility spaces. Moreover, it is shown that both are \emph{atomistic} and therefore they can be embedded in set algebras.
\end{abstract}

\tableofcontents


\section{Introduction and Overview}

In a recent paper~\cite{mirzaffalon20} derived some results about consistency and compatibility of coherent sets of gambles and coherent lower previsions and remarked that these results were in fact results of the theory of information or valuation algebras~\citep{kohlas03}. This point of view, however, was not worked out by~\cite{mirzaffalon20}. In this paper, this issue is addressed and it is shown that coherent sets of gambles, strictly desirable sets of gambles and coherent lower and upper previsions form indeed idempotent information algebras.  Like in group theory certain results concerning particular groups follow from the general theory, so many known results about coherent sets of gambles and coherent lower and upper previsions are indeed properties of information algebras and follow from the corresponding general theory. Some of these results are discussed in this paper but there are doubtless many others that can be derived.

From the point of view of information algebras, sets of gambles or lower previsions are pieces of information about certain groups of variables with their respective sets of possibilities. Such pieces of information can be aggregated or combined and the information they contain about sub-groups of variables can be extracted. This leads to an algebraic structure satisfying a number of simple axioms. In fact, there are two different versions of information algebras, a domain-free one and a labeled one. They are closely related and each one can be derived or reconstructed form the other. Grossly, the domain-free version is better suited for theoretical studies, since it is a structure of universal algebra, whereas the labeled version is better adapted to computational purposes, since it provides more efficient storage structures. In fact, labeled information algebras (or valuation algebras if idempotency is dropped) are the universal algebraic structures for local computation in join or junction trees as originally proposed for probabilistic networks by~\cite{lauritzenspiegelhalter88}. Based on this work,~\cite{shenoyshafer90} proposed a first axiomatic scheme which was sufficient to generalize the Lauritzen-Spiegelhalter scheme to a multitude of other uncertainty formalisms like Dempster-Shafer belief functions, possibility theory and many others. In~\cite{kohlas03} the algebraic theory is systematically developed and studied. In particular, the domain-free and labeled views are presented (based originally on~\citealp{shafer91}). In this paper both the domain-free and the labeled view of information algebra of coherent sets of gambles and coherent lower previsions are presented.

Coherent sets of gambles are analogous to logical theories because they can be considered as the result of applying a consequence operator to an original set of logical statements. In this case the consequence operator considered is the natural extension. It has been shown in~\cite{kohlas03} that a consequence operator, satisfying some mild conditions, induces an idempotent information algebra of logically closed sets. This theory applies to sets of desirable gambles if they avoid partial loss. The construction of the domain-free information algebra of coherent sets of gambles is based on this theory and presented in Section~\ref{secDomFreeInfAlg}. Based on the domain-free view, in Section~\ref{sec:LabInfAlg} the labeled version is derived by a standard construction for information algebras. Alternatively a second isomorphic version of it is presented, which corresponds more to the notion of marginals as used with sets of desirable gambles and lower previsions. In Section~\ref{sec:LowUpPrev} we show that there is also an information algebra of coherent lower previsions, and one of coherent upper previsions. 
Moreover, strictly desirable sets of gambles form a subalgebra of the algebra of coherent sets of gambles and this subalgebra is isomorphic to the algebras of coherent lower and upper previsions. 

Any idempotent information algebra induces an information order, a partial order, where one piece of information is less informative than another, if the combination of the two pieces yields the second one, i.e. the first piece of information is `contained' in the second one. In the case of coherent sets of gambles, this order corresponds to set inclusion and in the case of lower previsions to domination. Moreover, since natural extension of desirable sets of gambles is a consequence (or closure) operator, coherent sets of gambles form a complete lattice under information order, where meet is set intersection, and join is essentially natural extension. This carries over to coherent lower previsions where meet is the point-wise infimum of lower previsions and join again natural extension (Section~\ref{sec:LowUpPrev}). Furthermore, it is well know that there are (in this order) maximal coherent sets~\citep{CooQua12}. In the language of information algebra such maximal elements are called atoms. Any coherent set of gambles is contained in a maximal set (an atom) and it is the intersection (meet) of all atoms it is contained in. An information algebra with these properties is called atomistic. In fact, any intersection of atoms yields a coherent set of gambles. This is discussed in Section~\ref{sec:Atoms}. In the case of lower previsions, atoms are linear previsions (Section~\ref{sec:LinPrev}) and the fact that any coherent lower prevision is the meet of atoms it is contained in, is the lower envelope theorem.

Atomistic information algebras have the universal property that they are embedded in a set algebra, which is an information algebra the elements of which are sets, combination is simply intersection and extraction is the cylindrification or saturation operator. This is an important representation theorem for information algebras, since set algebras are very special kinds of algebras based on the usual set operations. Conversely, any such set algebra of subsets of the possibility space is embedded in the algebra of coherent sets of gambles. This has some importance for conditioning, a subject which however is not pursued in this paper. These links between set algebras and algebras of coherent sets of gambles are discussed in Section~\ref{sec:SetAlgs} and in Section~\ref{sec:LinPrev}.

As stated above, information algebras are the generic structures of local computation in join tree. This is used in Section~\ref{sec:Comp} to reconstruct the proof of~\cite{mirzaffalon20} concerning the running intersection property (RIP) as a condition for pairwise compatibility of elements (here coherent sets of gambles or coherent lower previsions) to be sufficient for global compatibility, which is to be marginals of some element.

\section{Desirability} \label{sec:DesGambles}

Consider a set $\Omega$ of possible worlds. A gamble over this set is a bounded function
\begin{eqnarray*}
f : \Omega \rightarrow \mathbb{R}.
\end{eqnarray*}

A gamble is interpreted as an uncertain reward in a linear utility scale. A subject might desire a gamble or not, depending on the information they have about the experiment whose possible outcomes are the elements of $\Omega$.

We denote the set of all gambles on $\Omega$ by $\mathcal{L}(\Omega)$, or more simply by $\mathcal{L}$ when there is no possible ambiguity. We also let $\mathcal{L}^+(\Omega) \coloneqq \{ f \in \mathcal{L}(\Omega): \; f\geq 0, f \not= 0\}$, or simply $\mathcal{L}^+$, denote the subset of non-vanishing, non-negative gambles. These gambles should always be desired, since they may increase the wealth with no risk of decreasing it.
As a consequence of the linearity of our utility scale, we assume also that a subject disposed to accept the transactions represented by $f$ and $g$, is disposed to accept also $\lambda f + \mu g$ with $\lambda, \mu \ge 0$ not both equal to $0$.
More generally, we can consider the notion of a coherent set of desirable gambles over $\Omega$ or, more simply, a coherent set of gambles over $\Omega$~\citep{walley91}.
\begin{definition}[\textbf{Coherent set of desirable gambles}]
We say that a subset $D$ of $\gambles(\Omega)$ is a \emph{coherent} set of desirable gambles, or more simply a coherent set of gambles, if and only if $D$ satisfies the following properties:
\begin{enumerate}[label=\upshape D\arabic*.,ref=\upshape D\arabic*]
\item\label{D1} $\gambles^+ \subseteq D$ [Accepting Partial Gains],
\item\label{D2} $0\notin D$ [Avoiding Null Gain],
\item\label{D3} $f,g \in D \Rightarrow f+g \in D$ [Additivity],
\item\label{D4} $f \in D, \lambda>0 \Rightarrow \lambda f \in D$ [Positive Homogeneity].
\end{enumerate}
\end{definition}
So $D$ is a convex cone.
This definition leads to the concept of natural extension.
\begin{definition}[\bf{Natural extension for gambles}] \label{def:natex}Given a set $D \subseteq\gambles(\Omega)$, we call $
\mathcal{E}(D) \coloneqq\posi(D\cup\gambles^+(\Omega))$,
where
\begin{equation*}
\posi(D')\coloneqq\left\{ \sum_{j=1}^{r} \lambda_{j}f_{j}: f_{j} \in D', \lambda_{j} > 0, r \ge 1\right\}
\end{equation*}
for every set $D' \subseteq \gambles(\Omega)$, its \emph{natural extension}.
\end{definition}

We do not indicate the dependency of the natural extension operator from the set $\Omega$. However, it has to be intended that, the operator applied to a set of gambles $D$, changes with the possibility set on which $D$ is defined. This fact becomes particularly important in Section~\ref{sec:LabInfAlg}.

The natural extension of a set of gambles $D$,  $\mathcal{E}(D)$, is coherent if and only if $0 \not\in \mathcal{E}(D)$. 

Coherent sets are closed under intersection, that is they form a topless $\cap$-structure~\citep{daveypriestley97}. By standard order theory~\citep{daveypriestley97} they are ordered by inclusion, intersection is meet in this order and a family of coherent sets of gambles $D_i,i \in I$, where $I$ is an index set, have a join if they have an upper bound among coherent sets
\begin{eqnarray*}
\bigvee_{i \in I} D_i \coloneqq \bigcap \{D' \in C(\Omega) : \bigcup_{i \in I} D_i \subseteq D'\},
\end{eqnarray*}
where we indicate with $C(\Omega)$, or simply with $C$ when there is no possible ambiguity, the family of coherent sets of gambles on $\Omega$.

Also, if $\mathcal{E}(D)$ is coherent, it is the smallest coherent set containing $D$
\begin{eqnarray*}
\mathcal{E}(D) = \bigcap \{D' \in C(\Omega): D \subseteq D'\},
\end{eqnarray*}
so that, if $\mathcal{E}(\bigcup_{i \in I} D_i)$ is coherent
\begin{eqnarray*}
\bigvee_{i \in I} D_i = \mathcal{E}(\bigcup_{i \in I} D_i).
\end{eqnarray*}

In view of the following section, it is convenient to add $\mathcal{L}(\Omega)$ to $C(\Omega)$ and let $\Phi(\Omega) \coloneqq C(\Omega) \cup \{\mathcal{L}(\Omega)\}$. In what follows, we refer to it also with $\Phi$, when there is no ambiguity. The family of sets in $\Phi$ is still a $\cap$-structure, but now a topped one. So, again by standard results of order theory~\citep{daveypriestley97}, $\Phi$ is a complete lattice under inclusion, meet is intersection and join is defined for any family of sets $D_i \in \Phi$ (for all $i \in I$) as
\begin{eqnarray*}
\bigvee_{i \in I} D_i \coloneqq \bigcap \{D' \in \Phi: \bigcup_{i \in I} D_i \subseteq D'\}.
\end{eqnarray*}
Note that, if a family of coherent sets $D_i$ has no upper bound in $C(\Omega)$, then its join is simply $\mathcal{L}(\Omega)$. In this topped $\cap$-structure, the operator
\begin{eqnarray*}
\mathcal{C}(D) \coloneqq \bigcap \{D' \in \Phi(\Omega): D \subseteq D'\}
\end{eqnarray*}
is a closure (or consequence) operator on $(\mathcal{P}(\gambles(\Omega)), \subseteq)$~\citep{daveypriestley97}. Also in this case, it changes with the possibility set on which $D$ is defined.
\begin{definition}[Closure operator on $(\mathcal{P}(\gambles), \subseteq)$]
A closure operator on the ordered set $(\mathcal{P}(\gambles), \subseteq)$ is a function $\mathcal{C}: \mathcal{P}(\gambles) \rightarrow \mathcal{P}(\gambles)$ that satisfies the following conditions for all sets  $D,D' \subseteq \mathcal{L}$:
\begin{itemize}
\item $D \subseteq \mathcal{C}(D)$,
\item $D \subseteq D'$ implies $\mathcal{C}(D) \subseteq \mathcal{C}(D')$,
\item $\mathcal{C}(\mathcal{C}(D)) = \mathcal{C}(D)$.
\end{itemize}
\end{definition}
Given that we always consider inclusion as the order relation on $\mathcal{P}(\gambles)$, we will refer to it, more simply, as a consequence operator on $\mathcal{P}(\gambles)$.

Note that $\mathcal{C}(D) = \mathcal{E}(D)$, if $0 \not\in \mathcal{E}(D)$, that is if $\mathcal{E}(D)$ is coherent. Otherwise we may have $\mathcal{E}(D) \not= \mathcal{L}(\Omega)$.  We refer to~\cite{decooman2005} for a similar order-theoretic view on belief models. These results prepare the way to an information algebra of coherent sets of gambles (see next section).  For further reference, note the following well-known result on consequence operators.

\begin{lemma} \label{le:UnionConsOp}
If $\mathcal{C}$ is a consequence operator on $\mathcal{P}(\gambles)$ then, for any $D_1, D_2 \subseteq \mathcal{L}(\Omega)$:
\begin{eqnarray*}
\mathcal{C}(D_1 \cup D_2) = \mathcal{C}(\mathcal{C}(D_1) \cup D_2).
\end{eqnarray*}
\end{lemma}

\begin{proof}
Obviously, $\mathcal{C}(D_1 \cup D_2) \subseteq \mathcal{C}(\mathcal{C}(D_1) \cup D_2)$. On the other hand $D_1,D_2 \subseteq D_1 \cup D_2$, hence $\mathcal{C}(D_1) \subseteq \mathcal{C}(D_1 \cup D_2)$ and $D_2 \subseteq \mathcal{C}(D_1 \cup D_2)$. This implies $\mathcal{C}(\mathcal{C}(D_1) \cup D_2) \subseteq \mathcal{C}(\mathcal{C}(D_1 \cup D_2)) = \mathcal{C}(D_1 \cup D_2)$.
\end{proof}


The most informative cases of coherent sets of gambles, i.e., coherent sets that are not proper subsets of other coherent sets, are called \textit{maximal}.
The following proposition provides a characterisation of such maximal elements~\citep{CooQua12}.
\begin{proposition}[\textbf{Maximal set of gambles}]
A coherent set of gambles $D$ is \emph{maximal} if and only if
\begin{equation*}
(\forall f \in \gambles \setminus \{0\})\ f \notin D \Rightarrow -f \in D.
\end{equation*}
\end{proposition} 
We shall indicate maximal sets of gambles with $M$ to differentiate them from the general case of coherent sets.
These sets play an important role because of the following facts proved in~\cite{CooQua12}:
\begin{enumerate}
\item any coherent set of gambles is a subset of a maximal one,
\item any coherent set of gambles is the intersection of all maximal coherent sets it is contained in.
\end{enumerate}
For a discussion of the importance of maximal coherent sets of gambles in relation to information algebras, see Section~\ref{sec:Atoms}.

A further important class of coherent sets of gambles are the \textit{strictly desirable} ones. 
\begin{definition} [\textbf{Strictly desirable set of gambles}]
A coherent set of gambles $D$ is said to be \emph{strictly desirable} if and only if it satisfies
$(\forall f  \in D \setminus \gambles^+)(\exists \delta >0)\ f- \delta \in D$.
\end{definition}
We shall employ the notation $D^+$ for strictly desirable sets of gambles to differentiate them from the general case of coherent sets of gambles. 

So strictly desirable sets of gambles are coherent and form a subfamily of coherent sets of gambles. In what follows we will indicate with $C^+(\Omega)$, or simply $C^+$, the family of strictly desirable sets of gambles. Moreover, analogously as before, we define 
$\Phi^+(\Omega) \coloneqq C^+(\Omega) \cup \{\mathcal{L}(\Omega)\}$, which we can indicate also with $\Phi^+$ when there is no ambiguity.

Another important class of sets, which plays an important role in Section~\ref{sec:LowUpPrev}, is the class of \textit{almost desirable sets of gambles}~\citep{walley91}.

\begin{definition}[\textbf{Almost desirable set of gambles}]
We say that a subset $\overline{D}$ of $\gambles(\Omega)$ is an \emph{almost desirable} set of gambles if and only if $\overline{D}$ satisfies the following properties:
\begin{enumerate}[label=\upshape D\arabic*'.,ref=\upshape D\arabic*']
\item\label{D1'}  $\inf f > 0$ implies $f \in \overline{D}$ [Accepting Sure Gains],
\item\label{D2'} $f \in \overline{D}$ implies $\sup f \geq 0$ [Avoiding Sure Loss],
\item\label{D3'} $f,g \in \overline{D} \Rightarrow f+g \in \overline{D}$ [Additivity],
\item\label{D4'} $f \in \overline{D}, \lambda>0 \Rightarrow \lambda f \in \overline{D}$ [Positive Homogeneity],
\item\label{D5'} $f + \delta \in \overline{D}$ for all $\delta > 0$ implies $f \in \overline{D}$ [Closure].
\end{enumerate}
\end{definition}
Such a set is no more coherent since it contains $f = 0$. We shall employ the notation $\overline{D}$ for almost desirable sets of gambles to differentiate them from coherent ones.

Let us indicate with $\overline{C}(\Omega)$, or simply $\overline{C}$, the family of almost desirable sets of gambles.
We remark that $\overline{C}(\Omega)$ again forms a $\cap$-system, still topped by $\mathcal{L}(\Omega)$.
Therefore if we add, as before, $\mathcal{L}(\Omega)$ to  $\overline{C}(\Omega)$ then $\overline{\Phi}(\Omega) \coloneqq \overline{C}(\Omega) \cup \{\mathcal{L}(\Omega)\}$, which we can indicate also with $\overline{\Phi}$, is a complete lattice too.

So we may define the operator
\begin{eqnarray*}
\overline{\mathcal{C}}(D) \coloneqq \bigcap \{\overline{D'} \in \overline{\Phi} :D \subseteq \overline{D'}\},
\end{eqnarray*}
that  is still a closure operator on sets of gambles.

Coherent set of gambles, strictly desirable sets of gambles and almost desirable ones encompass also a probabilistic model for $\Omega$, made of lower and upper expectations, called previsions after De Finetti~\citep{walley91, troffaes2014}.
We examine them as well from the point of view of information algebras in Section~\ref{sec:LowUpPrev}.

So far we have considered sets of gambles in $\mathcal{L}(\Omega)$ relative to a fixed set of possibilities $\Omega$. Next some more structure among sets of possibilities is introduced related to a family of variables and their possible values. This is the base for relating coherent sets of gambles to information algebras.


\section{Domain-Free Information Algebra} \label{secDomFreeInfAlg}

We may consider a coherent set of gambles as a piece of information about $\Omega$. 
We limit the analysis to the case in which information one is interested in concerns the values of certain groups of
variables. In this optic, coherent sets $D$ represent someone's evaluations of the `chances' of the elements of $\Omega$ to be possible answers to questions regarding them.
We assume therefore, a special form for the possibility space $\Omega$, namely a multivariate model. 

Let $I$ be an index set of variables $X_i, i \in I$. In practice, $I$ is often assumed to be
finite or countable. But we need not make this restriction. Any variable $X_i$ has a domain of possible values $\Omega_i$. For any subset $S$ of $I$ let
\begin{eqnarray*}
\Omega_S \coloneqq \bigtimes_{i \in S} \Omega_i,
\end{eqnarray*}
and $\Omega = \Omega_I$.\footnote{If needed, we assume the axiom of choice.}
We consider coherent sets of gambles on $\Omega$, or rather $\Phi$. The elements $\omega$ in $\Omega$ can be seen as functions $\omega : I \rightarrow \Omega$, so that $\omega_i \in \Omega_i$, for any $i \in I$. A gamble $f$ on $\Omega$ is called $S$-measurable, if for all $\omega,\omega' \in \Omega$ with $\omega \vert S = \omega' \vert S$ we have $f(\omega) = f(\omega')$ (here $\omega \vert S$ is the restriction of the map $\omega$ to $S$). Let $\mathcal{L}_S(\Omega)$, or more simply $\mathcal{L}_S$, denote the set of all $S$-measurable gambles. If $I= \emptyset$, $\gambles(\Omega_I) = \mathbb{R}$, the set of constant gambles~\citep[Section~2.3]{decooman2011}, moreover, we can define also $\mathcal{L}_\emptyset \coloneqq \mathbb{R}$, 
and note that $\mathcal{L}_I = \mathcal{L}(\Omega)$. For further reference  we have the following result.

\begin{lemma} \label{le:Measurability}
For any subsets $S$ and $T$ of $I$:
\begin{eqnarray*}
\mathcal{L}_{S \cap T} = \mathcal{L}_S \cap \mathcal{L}_T.
\end{eqnarray*}
\end{lemma}

\begin{proof}
If one of the two subsetes is empty we immediatly have the result. Otherwise, 
assume first $f \in \mathcal{L}_{S \cap T}$. Consider two elements $\omega,\mu \in \Omega$ so that $\omega \vert S = \mu \vert S$. Then we have also $\omega \vert S \cap T = \mu \vert S \cap T$ and $f(\omega) = f(\mu)$. So we see that $f \in \mathcal{L}_S$ and similarly $f \in \mathcal{L}_T$.

Conversely, assume $f \in \mathcal{L}_S \cap \mathcal{L}_T$. Consider two elements $\omega, \; \mu \in \Omega$. so that $\omega \vert S \cap T = \mu \vert S \cap T$. Consider then the element $\lambda \in \Omega$ defined as
\begin{eqnarray*}
\lambda_i \coloneqq \left\{ \begin{array}{ll} \omega_i = \mu_i, & i \in (S \cap T) , \\ \omega_i, & i \in (S \setminus T) \cup (S \cup T)^c, \\ \mu_i, & i \in T \setminus S \end{array} \right.
\end{eqnarray*}
for every $i \in I$.
Then $\lambda \vert S = \omega \vert S$ and $\lambda \vert T = \mu \vert T$. Since $f$ is both $S$- and $T$-measurable we have $f(\omega) = f(\mu)$. It follows that $f \in \mathcal{L}_{S \cap T}$ and this concludes the proof.
\end{proof}

Now we define in $\Phi$ the following operations.
\begin{enumerate}
\item Combination: $D_1 \cdot D_2 \coloneqq D_1 \vee D_2 \coloneqq \mathcal{C}(D_1 \cup D_2)$.
\item Extraction: $\epsilon_S(D) \coloneqq \mathcal{C}(D \cap \mathcal{L}_S)$, for every $S \subseteq I$.
\end{enumerate}
Let $\mathcal{E}_S(D) \coloneqq \mathcal{C}(D) \cap \mathcal{L}_S$, for any set of gambles $D$. Then, if $D$ is coherent, we have
\begin{eqnarray*}
\epsilon_S(D) = \mathcal{C}(\mathcal{E}_S(D)).
\end{eqnarray*}

If we consider coherent sets of gambles as pieces of information (i.e., someone's beliefs about an experiment with outcomes in $\Omega$), then combination represents the aggregation of two or more pieces of information and extraction the marginalization of information to a subset of variables. Note that $D_1 \cdot D_2 = \mathcal{L}(\Omega)$ for some $D_1, D_2 \in \Phi$, means that the two sets $D_1$ and $D_2$ are not \emph{consistent}, that is, $\mathcal{E}(D_1 \cup D_2)$ is not coherent (see Section~\ref{sec:Comp}). So, clearly $\mathcal{L}(\Omega)$ is the null element of combination and  represents \emph{inconsistency}. The set $\mathcal{L}^+(\Omega)$ instead, is the unit element of combination, representing vacuous information. We claim that $\Phi$ equipped with these two operations satisfies the properties collected in the following theorem.

\begin{theorem} \label{th:DomFreeInfAlg}
\begin{enumerate}
\item $(\Phi;\cdot)$ is a commutative semigroup with a null $0 = \gambles$ and unit $1= \gambles^+$.
\item For any subset $S \subseteq I$ and $D,D_1,D_2 \in \Phi$: 
\begin{description}
\item[E1] $\epsilon_S(0) = 0$,
\item[E2] $\epsilon_S(D) \cdot D = D$,
\item[E3] $\epsilon_S(\epsilon_S(D_1) \cdot D_2) = \epsilon_S(D_1) \cdot \epsilon_S(D_2)$.
\end{description}
\item For any $S,T \subseteq I$ and any $D \in \Phi$, $(\epsilon_S \circ \epsilon_T) (D) = (\epsilon_T \circ \epsilon_S) (D) = \epsilon_{S \cap T} (D)$.
\item For any $D \in \Phi$, $\epsilon_I(D)=D$.
\end{enumerate}
\end{theorem}
 
\begin{proof}
\begin{enumerate}
    \item That $(\Phi,\cdot)$ is a commutative semigroup follows from $D_1 \cdot D_2 \coloneqq D_1 \vee D_2$, for any $D_1,D_2$ in the complete lattice $\Phi$. 
    As stated above, $0 = \mathcal{L}(\Omega)$ is the null element and $1 = \mathcal{L}^+(\Omega)$ the unit element of the semigroup (null and unit in a semigroup are always unique). 
    \item For E1 we have
\begin{eqnarray*}
\epsilon_S(0) = \epsilon_S(\mathcal{L}(\Omega)) \coloneqq \mathcal{C}(\mathcal{L}(\Omega) \cap \mathcal{L}_S) = \mathcal{C}(\mathcal{L}_S) = \mathcal{L}(\Omega) = 0,
\end{eqnarray*}
for any $S \subseteq I$.

E2 follows since $D \cap \mathcal{L}_S \subseteq D$ and $C(D \cap \mathcal{L}_S) \subseteq D$, for any $D \in \Phi$, $S \subseteq I$.

To prove E3 define, using Lemma~\ref{le:UnionConsOp},
\begin{eqnarray*}
A &\coloneqq& \mathcal{C}(\mathcal{C}(D_1 \cap \mathcal{L}_S) \cup D_2) \cap \mathcal{L}_S = \mathcal{C}((D_1 \cap \mathcal{L}_S) \cup D_2) \cap \mathcal{L}_S, \\
B &\coloneqq& \mathcal{C}(\mathcal{C}(D_1 \cap \mathcal{L}_S) \cup (\mathcal{C}(D_2 \cap \mathcal{L}_S)) = \mathcal{C}((D_1 \cap \mathcal{L}_S) \cup (D_2 \cap \mathcal{L}_S)).
\end{eqnarray*}
Then we have $B \coloneqq \epsilon_S(D_1) \cdot \epsilon_S(D_2)$ and $\mathcal{C}(A) \coloneqq \epsilon_S(\epsilon_S(D_1) \cdot D_2)$. Note that $B \subseteq \mathcal{C}(A)$. 

We claim first that:
\begin{equation}
     \epsilon_S(D_1) \cdot \epsilon_S(D_2) =0 \iff \epsilon_S(D_1) \cdot D_2 =  0.
\end{equation} Indeed, $\epsilon_S(D_1) \cdot \epsilon_S(D_2)=0$ implies a fortiori $\epsilon_S(D_1) \cdot D_2=0$. 

Assume therefore that $\epsilon_S(D_1) \cdot D_2 = 0$. 
This implies $0 = \mathcal{C}( \mathcal{C}(D_1 \cap \mathcal{L}_S) \cup D_2)= \mathcal{C}( (D_1 \cap \mathcal{L}_S) \cup D_2)$, by Lemma~\ref{le:UnionConsOp}. Now, if $D_1= \mathcal{L}(\Omega)$ or $D_2= \mathcal{L}(\Omega)$ we have immediately the result, otherwise we claim that $0 = f + g'$ with $f \in D_1 \cap \mathcal{L}_S$ and $g' \in D_2 \cap \mathcal{L}_S$.
Indeed, from $0 = \mathcal{C}( (D_1 \cap \mathcal{L}_S) \cup D_2)$, we know that $0 = f+g+h'$ with $f \in D_1 \cap \mathcal{L}_S$, $g \in D_2$, $h' \in \mathcal{L}^+(\Omega) \subseteq D_2$ or $h' =0$. Then, if we introduce $g'=g+h'$, we have $0 = f+g'$ with $f \in D_1 \cap \mathcal{L}_S$, $g' \in D_2$. However, this implies $g'= -f \in \mathcal{L}_S$ and then $g' \in D_2 \cap \mathcal{L}_S$.
Notice that $\epsilon_S(D_1) \cdot \epsilon_S(D_2) \eqqcolon B = \mathcal{C}((D_1 \cap \mathcal{L}_S) \cup (D_2 \cap \mathcal{L}_S))$. Therefore, we have the result.


So, we may now assume that both $\epsilon_S(D_1) \cdot D_2$ and $\epsilon_S(D_1) \cdot \epsilon_S(D_2)$ are coherent. Then we have
\begin{eqnarray*}
A = \{f \in \mathcal{L}_S:f \geq \lambda g + \mu h,g \in D_1 \cap \mathcal{L}_S,h \in D_2,\lambda,\mu \geq 0,f \not= 0\}.
\end{eqnarray*}

Consider $f \in A$. Then $f = \lambda g + \mu h + h'$, where $h' \in \mathcal{L}^+(\Omega) \cup \{0\}$. Since $f$ and $g$ are $S$-measurable, $\mu h + h'$ must be $S$-measurable.
Now, if $\mu h + h'=0$ then $f \in D_1 \cap \gambles_S \subseteq B$. Otherwise,
$\mu h + h' \in D_2 \cap \mathcal{L}_S$. So in any case $f \in B$, hence we have $\mathcal{C}(A) \subseteq C(B)=B$.
\item Note first that $\epsilon_S \circ \epsilon_T(D) = 0$ and $\epsilon_{S \cap T}(D) = 0$ if and only if $D = 0$. So assume $D$ to be coherent. Then we have
\begin{eqnarray*}
\epsilon_S(\epsilon_T(D)) &\coloneqq& \mathcal{C}(C(D \cap \mathcal{L}_T) \cap \mathcal{L}_S), \\
\epsilon_{S \cap T}(D) &\coloneqq& \mathcal{C}(D \cap \mathcal{L}_{S \cap T}) = \mathcal{C}(D \cap \mathcal{L}_T \cap \mathcal{L}_S).
\end{eqnarray*}
Obviously, $\epsilon_{S \cap T}(D) \subseteq \epsilon_S(\epsilon_T(D))$. Consider then $f \in C(D \cap \mathcal{L}_T) \cap \mathcal{L}_S$. If $f \in \gambles^+_S$ then clearly $f \in \epsilon_{S \cap T}(D) $. Otherwise, 
\begin{eqnarray*}
f \in \mathcal{L}_S, \quad f \geq g,\quad g \in D \cap \mathcal{L}_T.
\end{eqnarray*}
Define
\begin{eqnarray*}
g'(\omega) \coloneqq \sup_{\lambda \vert S = \omega \vert S} g(\lambda).
\end{eqnarray*}
Then we have $f \geq g'$. Clearly $g'$ is $S$-measurable and belongs to $D$, $g' \in D \cap \mathcal{L}_S$. We claim that $g'$ is also $T$-measurable. Consider two elements $\omega$ and $\mu$ so that $\omega \vert S \cap T = \mu \vert S \cap T$. Note that we may write
\begin{eqnarray*}
g'(\omega) \coloneqq \sup_{\lambda \vert S = \omega \vert S} g(\lambda) = \sup_{\lambda \vert I \setminus S} g(\omega \vert S \cap T,\omega \vert S \setminus T,\lambda \vert T \setminus S,\lambda \vert R),
\end{eqnarray*}
where $R = (S \cup T)^c$. Similarly, we have
\begin{eqnarray*}
g'(\mu) \coloneqq \sup_{\lambda' \vert S = \mu \vert S} g(\lambda') = \sup_{\lambda' \vert I \setminus S} g(\omega \vert S \cap T,\mu \vert S \setminus T,\lambda' \vert T \setminus S,\lambda' \vert R).
\end{eqnarray*}

Since $g$ is $T$-measurable, we have:
\begin{eqnarray*}
g'(\mu) = \sup_{\lambda' \vert I \setminus S} g(\omega \vert S \cap T,\omega \vert S \setminus T,\lambda' \vert T \setminus S,\lambda' \vert R),
\end{eqnarray*}
that clearly coincides with $g'(\omega)$.

This shows that $g'$ is $S \cap T$-measurable, therefore both $S$- and $T$-measurable by Lemma~\ref{le:Measurability}. So we have $g' \in D \cap \mathcal{L}_S \cap \mathcal{L}_T$, hence $f \in \mathcal{C}(D \cap \mathcal{L}_T \cap \mathcal{L}_S)$. 

\item It is obvious.
\end{enumerate}

\end{proof}

A system with two operations satisfying the conditions of Theorem~\ref{th:DomFreeInfAlg} is called a \emph{domain-free information algebra}~\citep{kohlas03}, or more precisely a \emph{commutative domain-free information algebra}~\citep{kohlasschmid16}. In this case, with a little abuse of notation, we call $\Phi$ with the two operations defined above, the \emph{domain-free information algebra of coherent sets of gambles}. There is an alternative, equivalent version, a so-called \emph{labeled information algebra}, which may derived from it, see the next section. In any information algebra an \textit{information order} can be introduced. In the case of coherent sets of gambles, this order is defined as $D_1 \leq D_2$ if $D_1 \cdot D_2 = D_2$. This means that $D_2$ is more informative than $D_1$ if adding $D_1$ gives nothing new; $D_1$ is already contained in $D_2$. It is easy to verify that it is a partial order and also that $D_1 \leq D_2$ if and only if $D_1 \subseteq D_2$. 
So, in particular, $(\Phi,\leq)$ is a complete lattice under information order, since information order corresponds to set inclusion. Moreover,  combination corresponds to join, $D_1 \cdot D_2 = D_1 \vee D_2$, vacuous information $\gambles^+(\Omega)$ is the least information in this order and $\gambles(\Omega)$ the top element (although strictly speaking it is no more a piece of information, since it represents inconsistency).
Conditions E1 to E3 in Theorem~\ref{th:DomFreeInfAlg} can also be rewritten using this order as the following.
For any subset $S \subseteq I$ and $D,D_1,D_2 \in \Phi$:
\begin{description}
\item[E1] $\epsilon_S(0) = 0$,
\item[E2] $\epsilon_S(D) \leq D$, 
\item[E3] $\epsilon_S(\epsilon_S(D_1) \vee D_2) = \epsilon_S(D_1) \vee \epsilon_S(D_2)$.
\end{description}
In algebraic logic such an operator is also called an \textit{existential quantifier}.\footnote{Although usually operators on a Boolean lattice are considered and the order is inverse to the information order.}

We further claim that extraction distributes over intersection (or meet in the complete lattice).

\begin{theorem} \label{th:ExtrMeet}
Let  $D_j$, $j \in J$ be any family of sets of gambles in $\Phi$ and $S$ any subset of variables. Then
\begin{eqnarray} \label{eq:ExzrMeet}
\epsilon_S(\bigcap_{j \in J} D_j) = \bigcap_{j \in J} \epsilon_S(D_j).
\end{eqnarray}
\end{theorem}

\begin{proof}
We may assume that $D_j \in C(\Omega)$ for all $j \in J$ since if some or all $D_j = \mathcal{L}(\Omega)$, then we may restrict the intersection on both sides over the set $D_j \in C(\Omega)$, or the intersection over both sides equals $\mathcal{L}(\Omega)$. If $D_j \in C(\Omega)$ for all $j \in J$, we have 
\begin{eqnarray*}
\epsilon_S(\bigcap_{j \in J} D_j) &=& \mathcal{E}((\bigcap_{j \in J} D_j) \cap \mathcal{L}_S) = \posi(\mathcal{L}^+(\Omega) \cup ((\bigcap_{j \in J} D_j) \cap \mathcal{L}_S)), \\
\bigcap_{j \in J} \epsilon_S(D_j) &=& \bigcap_{j \in J} \mathcal{E}(D_j \cap \mathcal{L}_S) = \bigcap_{j \in J} \posi(\mathcal{L}^+(\Omega) \cup (D_j \cap \mathcal{L}_S)).
\end{eqnarray*}
Consider first a gamble $f \in \epsilon_S(\bigcap_{j \in J} D_j)$ so that $f = \lambda g + \mu h$, where $\lambda,\mu$ nonnegative and not both equal zero, $g \in (\bigcap_{j \in J} D_j) \cap \mathcal{L}_S$ and $h \in \mathcal{L}^+(\Omega)$. But then $g \in D_j \cap \mathcal{L}_S$ for all $j \in J$, so that $f \in  \bigcap_{j \in J} \epsilon_S(D_j)$.

Conversely, assume $f \in  \bigcap_{j \in J} \epsilon_S(D_j)$. If $f \in \gambles^+$, then $f \in \epsilon_S(\bigcap_{j\in J} D_j)$. Otherwise, $f \ge g_j$ for some $g_j \in (D_j \cap \mathcal{L}_S)$, for all $j \in J$.
Hence, $f(\omega) \ge \sup_{k \in J} g_k (\omega)$ for every $\omega \in \Omega$.
However, $\sup_{k \in J} g_k \in \bigcap_{j \in J} (D_j) $ because $\sup_{k \in J} g_k(\omega) \ge g_j(\omega)$ for all $j \in J$, for all $\omega \in \Omega$. Moreover, $\sup_{k \in J} g_k \in \mathcal{L}_S$, thanks to the fact that $g_j \in \mathcal{L}_S$ for all $j \in J$. Therefore, $f \in \epsilon_S(\bigcap_{j \in J} D_j)$.



\end{proof}

So $(\Phi,\leq)$ is a lattice under information order and satisfies Eq.~\eqref{eq:ExzrMeet}. Such an information algebra is called a \emph{lattice} information algebra.


The family of strictly desirable sets of gambles enlarged with $\mathcal{L}(\Omega)$ is also closed under combination and extraction in $\Phi$. Therefore 
$\Phi^+$ forms a \emph{subalgebra} of $\Phi$.



\section{Labeled Information Algebra} \label{sec:LabInfAlg}

The domain-free information algebra of coherent sets of gambles treats the general case of gambles defined on $\Omega$. However, it is well known that, if a coherent set of gambles $D$ is such that $D= \edesirs(D \cap \gambles_S)$ for some $S \subset I$, it is essentially determined by values of gambles defined on smaller sets of possibilities than $\Omega$, namely on blocks $[\omega]_S$ of the equivalence relation $\equiv_S$ defined as $\omega \equiv_S \omega' \iff \omega \vert S = \omega' \vert S$ for every $\omega, \omega' \in \Omega$. 
Indeed, it contains the same information of the set $D \cap \gambles_S$ that is in a one-to-one correspondence with a set $\tilde{D}$ directly defined on 
blocks $[\omega]_S$ (see for example~\citealp{mirzaffalon20}).
This view leads to another, so-called labeled version of information algebras that clearly is better suited for computational purposes. 

We start deriving a labeled view of the information algebra of coherent sets of gambles, using a general method for domain-free information algebras to derive corresponding labeled ones~\citep{kohlas03}.
In this case, as well as in the case of coherent lower previsions, there is a second isomorphic version of the labeled algebra, which is nearer to the intuition explained before and which will be introduced after the general construction of the labeled algebra. From a labeled information algebra, the domain-free one may be reconstructed~\citep{kohlas03}. So the two views are equivalent. 

The whole theory presented here could also have been developed in the labeled view. It is a question of convenience whether the domain-free or the labeled view is chosen. Usually, for theoretical considerations, the domain-free view is preferable because it is nearer to universal algebra. For computational purposes the labeled view is used in general, because it corresponds better to the needs of efficient data representation.

We begin introducing the concept of support.
\begin{definition}[Support for sets of gambles]
A subset $S$ of $I$ is called \emph{support} of a set of gambles $D \in \Phi$, if $\epsilon_S(D) = D$.
\end{definition}
This means that the information contained in $D$ concerns, or is focused on, the group $S$ of variables. Here are a few well-known results on  supports in domain-free information algebras (for proofs see~\citeauthor{kohlas03}, \citeyear{kohlas03}).
\begin{lemma}
Let $D, D_1, D_2 \in \Phi, \; S,T \subseteq I$. The following are true:
\begin{enumerate}
\item any $S$ is a support of the null $0$ ($\mathcal{L}(\Omega))$ and the unit $1$ ($\mathcal{L}^+(\Omega))$ elements,
\item $S $ is a support of $\epsilon_S(D)$,
\item if $S$ is a support of $D$, then $S$ is a support of $\epsilon_T(D)$,   
\item if $S$ and $T$ are supports of $D$, then so is $S \cap T$,
\item if $S$ is a support of $D$, then $\epsilon_T(D) = \epsilon_{S \cap T}(D)$, 
\item if $S$ is a support of $D$ and $S \subseteq T$, then $T$ is a support of $D$, 
\item if $S$ is a support of $D_1$ and $D_2$, then it is also a support of $D_1 \cdot D_2$, 
\item if $S$ is a support of $D_1$ and $T$ a support of $D_2$, then $S \cup T$ is a support of $D_1 \cdot D_2$. 
\end{enumerate}
\end{lemma}
Now, we consider sets $\Psi_S(\Omega)$, or $\Psi_S$ when there is no possible ambiguity, of pairs $(D,S)$, where $S \subseteq I$ is a support of $D \in \Phi$. That is, we collect together pieces of information concerning the same set of variables. Let
\begin{eqnarray*}
\Psi(\Omega) \coloneqq \bigcup_{S \subseteq I} \Psi_S(\Omega).
\end{eqnarray*}
As usual, we can refer to it also with $\Psi$, when there is no possible ambiguity.
In $\Psi$ we define the following operations in terms of the ones of the domain-free algebra.
\begin{enumerate}
\item Labeling: $d(D,S) \coloneqq S$.
\item Combination: $(D_1,S) \cdot (D_2,T) \coloneqq (D_1 \cdot D_2,S \cup T)$.
\item Projection (Marginalization): $\pi_T(D,S) \coloneqq (\epsilon_T(D),T)$, for every $T \subseteq S \subseteq I$. 
\end{enumerate}
It is well-known and easy to verify that $\Psi$ with these three operations satisfies the properties in the following theorem~\citep{kohlas03}.

\begin{theorem} \label{th:LabInfAlg}
\item Semigroup: $(\Psi,\cdot)$ is a commutative semigroup.
\item Labeling: $d((D_1,S) \cdot (D_2,T)) = d(D_1,S) \cup d(D_2,T)$ , for any $(D_1,S), (D_2,T) \in \Psi$ and $d(\pi_T(D,S)) = T$, for any $(D,S) \in \Psi, \; T \subseteq S \subseteq I$.
\item Null and Unit: $(0,S) \cdot (D,S) = (0,S)$, $(1,S) \cdot (D,S) = (D,S)$ for any $(D,S) \in \Psi$, $\pi_T(D,S) = (0,T)$ if and only if $(D,S) = (0,S)$, $\pi_T(1,S) = (1,T)$, with $T \subseteq S \subseteq I$ and $(1,S) \cdot (1,T) = (1,S \cup T)$, with $S,T \subseteq I$.
\item Transitivity of Projection: if $U \subseteq T \subseteq S \subseteq I$, then $\pi_U(D,S) = \pi_U(\pi_T(D,S))$, for any $(D,S) \in \Psi$.
\item Combination: $\pi_S((D_1,S) \cdot (D_2,T)) = (D_1,S) \cdot \pi_{S \cap T}(D_2,T)$, for any $(D_1,S), (D_2,T) \in \Psi$.
\item Idempotency: $(D,S) \cdot \pi_T(D,S) = (D,S)$, for any $(D,S) \in \Psi, \; T \subseteq S \subseteq I$.
\item Identity: $\pi_S(D,S) = (D,S)$, for any $(D,S) \in \Psi$.
\end{theorem}

An algebraic system like $\Psi$ with the operations defined before, satisfying the conditions of Theorem~\ref{th:LabInfAlg} is called a \emph{labeled (idempotent and stable) information algebra}~\citep{kohlas03}. In this sense, coherent sets of gambles augmented with the set of all gambles $\gambles(\Omega)$ form both a domain-free and a labeled information algebra.  

We may associate to this labeled algebra another, isomorphic one. For a subset $S$ of variables, let $C(\Omega_S)$ be the family of coherent sets of gambles on $\Omega_S$. 
Furthermore, let $\tilde{\Psi}_S(\Omega)$ be the set of pairs $(\tilde{D},S)$, where $S \subseteq I$ and $\tilde{D} \in C(\Omega_S) \cup \{ \gambles(\Omega_S)\}$, 
and
\begin{eqnarray*}
\tilde{\Psi}(\Omega) \coloneqq \bigcup_{S \subseteq I} \tilde{\Psi}_S(\Omega).
\end{eqnarray*}
We will refer to them also with $\tilde{\Psi}$ and $\tilde{\Psi}_S$ for every $S \subseteq I$, when there is no ambiguity.

It is well known that there is a one-to-one correspondence between gambles $f \in \gambles_S(\Omega_{R})$ with $S \subseteq R \subseteq I$, and gambles $f' \in \gambles(\Omega_S)$. So, in what follows, given a gamble $f \in \gambles_S(\Omega_{R})$, we indicate with  $f^{\downarrow{S}}$ the corresponding gamble in $\gambles(\Omega_{S})$ defined, for all $\omega_S \in \Omega_S$, as $f^{\downarrow{S}}(\omega_S) \coloneqq f(\omega_{R})$ for all $\omega_{R} \in \Omega_R$
 such that $\omega_R \vert S=\omega_S$. 
Vice versa, given a gamble $f' \in \gambles(\Omega_S)$ we indicate with $(f')^{\uparrow{R}}$ the corresponding gamble in $\gambles_S(\Omega_{R})$ defined as $(f')^{\uparrow{R}}(\omega_{R}) \coloneqq f'(\omega_{R} \vert S)$, for all $\omega_{R} \in \Omega_{R}$.
Clearly, if $f \in \gambles_S(\Omega_R)$, then $(f^{\downarrow{S}})^{\uparrow{R}}= f$ and  $f^{\downarrow{S}}= (f^{\uparrow{I}})^{\downarrow{S}}$ and also $f^{\uparrow{I}}=(f^{\downarrow{S}})^{\uparrow{I}} $. Vice versa, if $f' \in \gambles(\Omega_S)$, then $((f')^{\uparrow{R}})^{\downarrow{S}}= f'= ((f')^{\uparrow{I}})^{\downarrow{S}}$.

We extend these maps also to sets of gambles in the following way:
\begin{eqnarray*}
D^{\downarrow{S}} \coloneqq \{ f' \in \gambles(\Omega_S): f'=f^{\downarrow{S}} \text{ for some } f \in D\},
\end{eqnarray*}
for every $D \subseteq \gambles_S(\Omega_{R})$.
\begin{eqnarray*}
D^{\uparrow{R}} \coloneqq \{ f \in \gambles_S(\Omega_{R}): f=(f')^{\uparrow{R}} \text{ for some } f' \in D\},
\end{eqnarray*}
for every $D \subseteq \gambles(\Omega_S)$.

\begin{lemma}\label{le:cylExtProperties}
Consider $T \subseteq S \subseteq R \subseteq I$. The following properties are valid.
\begin{enumerate}
     \item If $D \subseteq \gambles_T(\Omega_R)$, then $ D^{\downarrow{S}}= (D^{\downarrow{T}})^{\uparrow{S}} $. So, in particular, if $S=R$, we have $D = (D^{\downarrow{T}})^{\uparrow{R}}$.
    \item If $D \subseteq \gambles_T(\Omega_S)$, then $D^{\downarrow{T}}= (D^{\uparrow{R}})^{\downarrow{T}}$. 
    \item If $D_1,D_2 \subseteq \gambles_T(\Omega_R)$, then $D_1^{\downarrow{T}} \cap D_2^{\downarrow{T}} = (D_1 \cap D_2)^{\downarrow{T}}$.
    \item If $D_1,D_2 \subseteq \gambles_T(\Omega_R)$, then $D_1^{\downarrow{T}} \cup D_2^{\downarrow{T}} = (D_1 \cup D_2)^{\downarrow{T}}$.
    \item If $D \subseteq \gambles_T(\Omega_R)$, then $(\mathcal{C}(D) \cap \gambles_T)^{\downarrow{T}} = \mathcal{C}(D^{\downarrow{T}})$.
\end{enumerate}
\end{lemma}
\begin{proof}
Items 1,2,3 and 4 are obvious.
Regarding item 5, we will show that $(\edesirs(D) \cap \gambles_T)^{\downarrow{T}} = \edesirs(D^{\downarrow{T}})$, from which item 5 derives.
So, consider $f' \in (\edesirs(D) \cap \gambles_T)^{\downarrow{T}}$. Then $f' = f^{\downarrow{T}}$, for some $f \in \edesirs(D) \cap \gambles_T$, so, for every $\omega_T \in \Omega_T$, $f'(\omega_T) = f(\omega_R)= \sum_{i=1}^r \lambda_i g_i(\omega_R) + \mu h(\omega_R)$, with $\lambda_i, \mu \ge 0, \; \forall i$ not all equal to $0$, $r \ge 0$, $g_i \in D \subseteq \gambles_T(\Omega_R), \; h \in \gambles^+$, for every $\omega_R \in \Omega_R$ such that $\omega_R \vert T = \omega_T$. Therefore $h \in \gambles^+_T$. So, $f' = \sum_{i=1}^r \lambda_i g_i^{\downarrow{T}} + \mu h^{\downarrow{T}}$, therefore $f' \in \edesirs(D^{\downarrow{T}})$. The other inclusion can be proven analogously, therefore we have the thesis.

\end{proof}

Within $\tilde{\Psi}$ we define the following operations.
\begin{enumerate}
\item Labeling: $d(\tilde{D},S) \coloneqq S$.
\item Combination: $(\tilde{D}_1,S) \cdot (\tilde{D_2},T) \coloneqq ( \mathcal{C}( \tilde{D}_1^{\uparrow{S \cup T}})  \cdot \mathcal{C}( \tilde{D}_2^{\uparrow{S \cup T}}), S \cup T)$. 

\item Projection (Marginalization): $\pi_T(\tilde{D},S) \coloneqq ((\epsilon_T( \tilde{D}) \cap \gambles_T(\Omega_S))^{\downarrow{T}},T) 
$, for every $T \subseteq S \subseteq I$.
\end{enumerate}

Consider now the map $h : \Psi \rightarrow \tilde{\Psi}$ defined by $(D,S) \mapsto ((\epsilon_S(D) \cap \mathcal{L}_S(\Omega))^{\downarrow{S}},S)=((D \cap \mathcal{L}_S(\Omega))^{\downarrow{S}},S)$, because $D$ has support $S$. The map is clearly well defined. Moreover, $h$ will establish an isomorphism between the labeled information algebras $\Psi$ and $\tilde{\Psi}$ with the respective operations.


\begin{theorem} \label{th:LabelIsom}
The map $h$ has the following properties.
\begin{enumerate}
\item It maintains combination, null and unit, and projection. Let $(D,S), \; (D_1,S), \; (D_2,T) \in \Psi$:
\begin{eqnarray*}
h((D_1,S) \cdot (D_2,T)) &=& h(D_1,S) \cdot h(D_2,T), \\
 h(\mathcal{L}(\Omega),S) &=& (\mathcal{L}(\Omega_S),S), \\
 h(\mathcal{L}^+(\Omega),S) &=& (\mathcal{L}^+(\Omega_S),S), \\
h(\pi_T(D,S)) &=& \pi_T(h(D,S)), \textrm{ if } T  \subseteq S.
\end{eqnarray*}
\item $h$ is bijective.
\end{enumerate}
\end{theorem}

\begin{proof}
\begin{enumerate}
    \item Recalling that $D_1$ has support $S$ and $D_2$ has support $T$, we have by definition
    \begin{eqnarray*}
\lefteqn{h((D_1,S) \cdot (D_2,T)) \coloneqq h(D_1 \cdot D_2,S \cup T)}\\ &&\coloneqq (((D_1 \cdot D_2) \cap \mathcal{L}_{S \cup T})^{\downarrow{S \cup T}}, S \cup T) \\ && = ((\mathcal{C}((D_1 \cap \gambles_S) \cup (D_2 \cap \gambles_T)) \cap \mathcal{L}_{S \cup T})^{\downarrow{S \cup T}}, S \cup T) \\ && = (\mathcal{C}( ((D_1 \cap \gambles_S) \cup (D_2 \cap \gambles_T))^{\downarrow{S \cup T}}), S \cup T) 
\end{eqnarray*}
thanks to Lemma~\ref{le:UnionConsOp} and item 5 of Lemma~\ref{le:cylExtProperties}.
On the other hand, thanks again to Lemma~\ref{le:UnionConsOp}, we have
\begin{eqnarray*}
h(D_1,S) \cdot h(D_2,T) \coloneqq ((D_1 \cap \gambles_S)^{\downarrow{S}},S) \cdot ((D_2 \cap \gambles_T)^{\downarrow{T}},T) \coloneqq
\end{eqnarray*}
\begin{eqnarray*}
(\mathcal{C}( ((D_1 \cap \gambles_S)^{\downarrow{S}})^{\uparrow{S \cup T}} \cup ((D_2 \cap \gambles_T)^{\downarrow{T}})^{\uparrow{S \cup T}}), S \cup T).
\end{eqnarray*}
Now, using again properties of Lemma~\ref{le:cylExtProperties}, we have
\begin{eqnarray*}
\lefteqn{h(D_1,S) \cdot h(D_2,T) } \\
&& \coloneqq (\mathcal{C}( ((D_1 \cap \gambles_S)^{\downarrow{S}})^{\uparrow{S \cup T}} \cup ((D_2 \cap \gambles_T)^{\downarrow{T}})^{\uparrow{S \cup T}}), S \cup T) \\
&&= (\mathcal{C}( (D_1 \cap \gambles_S)^{\downarrow{S \cup T}} \cup (D_2 \cap \gambles_T)^{\downarrow{S \cup T}} ), S \cup T)
\\
&&= (\mathcal{C}( ((D_1 \cap \gambles_S) \cup (D_2 \cap \gambles_T))^{\downarrow{S \cup T}}), S \cup T) = h(D_1,S) \cdot h(D_2,T).
\end{eqnarray*}



Obviously, $(\mathcal{L}(\Omega),S)$ maps to $(\mathcal{L}(\Omega_S),S)$ and $(\mathcal{L}^+(\Omega),S)$ maps to $(\mathcal{L}^+(\Omega_S),S)$. Then we have, again by definition
\begin{eqnarray*}
\lefteqn{h(\pi_T(D,S)) \coloneqq  h(\epsilon_T(D),T)} \\&& \coloneqq ((\epsilon_T(D) \cap \mathcal{L}_T)^{\downarrow{T}}, T) \\ &&=  ( ( D \cap \mathcal{L}_T)^{\downarrow{T}}, T).
\end{eqnarray*}
Indeed, $D \cap \gambles_T \subseteq \mathcal{C}(D \cap \gambles_T) \cap \gambles_ T = \epsilon_T(D) \cap \gambles_T \subseteq \mathcal{C}(D) \cap \gambles_T = D \cap \gambles_T$.
However, from $T \subseteq S$, it follows $\gambles_T \subseteq \gambles_S$. Therefore we have
\begin{eqnarray*}
\lefteqn{h(\pi_T(D,S)) = ( ( D \cap \mathcal{L}_T)^{\downarrow{T}}, T)}  \\&&= ( (D \cap \gambles_S) \cap \mathcal{L}_T)^{\downarrow{T}},T).
\end{eqnarray*}
On the other hand, we have
\begin{eqnarray*}
\lefteqn{\pi_T(h(D,S)) \coloneqq \pi_T( (D \cap \gambles_S)^{\downarrow{S}},S)}
\\ && \coloneqq( (\epsilon_T((D \cap \gambles_S)^{\downarrow{S}}) \cap \gambles_T(\Omega_S))^{\downarrow{T}},T)
\\ && =(((D \cap \gambles_S)^{\downarrow{S}} \cap \gambles_T(\Omega_S))^{\downarrow{T}},T) \\ && = (((D \cap \gambles_S)^{\downarrow{S}} \cap (\gambles_T(\Omega))^{\downarrow{S}} )^{\downarrow{T}},T) \\ && = (( (D \cap \gambles_S) \cap \gambles_T)^{\downarrow{S}} )^{\downarrow{T}},T) \\ && = ( (( (D \cap \gambles_S) \cap \gambles_T)^{\downarrow{S}}) ^{\uparrow{I}} )^{\downarrow{T}},T) \\ && = ( (D \cap \gambles_S) \cap \mathcal{L}_T)^{\downarrow{T}},T)= h(\pi_T(D,S)),
\end{eqnarray*}
thanks to Lemma~\ref{le:cylExtProperties}.
\item  Suppose $h(D_1,S) =h(D_2,T)$. Then we have $S=T$ and $(D_1 \cap \gambles_S)^{\downarrow{S}} = (D_2 \cap \gambles_S)^{\downarrow{S}} $, from which we derive that $D_1 \cap \gambles_S = D_2 \cap \gambles_S$ and therefore, $D_1= \mathcal{C}(D_1 \cap \gambles_S)= \mathcal{C}(D_2 \cap \gambles_S)= D_2$.
So the map $h$ is injective.

Moreover, for any $(\tilde{D},S) \in \tilde{\Psi}$ we have
that $(\tilde{D},S) = h(D,S)$ where $(D,S)=(\mathcal{C}(\tilde{D}^{\uparrow{I}}),S) \in \Psi$. Indeed:
\begin{itemize}
    \item $(\mathcal{C}(\tilde{D}^{\uparrow{I}}),S) \in \Psi$. In fact, $\epsilon_S(\mathcal{C}(\tilde{D}^{\uparrow{I}})) \coloneqq \mathcal{C}(\mathcal{C}(\tilde{D}^{\uparrow{I}}) \cap \gambles_S)$. Now, $\tilde{D}^{\uparrow{I}} \subseteq \mathcal{C}(\tilde{D}^{\uparrow{I}}) \cap \gambles_S$, therefore $\mathcal{C}(\tilde{D}^{\uparrow{I}}) \subseteq \mathcal{C}(\mathcal{C}(\tilde{D}^{\uparrow{I}}) \cap \gambles_S)$. On the other hand, $\mathcal{C}(\tilde{D}^{\uparrow{I}}) \cap \gambles_S \subseteq \mathcal{C}(\tilde{D}^{\uparrow{I}})$, therefore $\mathcal{C}(\mathcal{C}(\tilde{D}^{\uparrow{I}}) \cap \gambles_S) \subseteq \mathcal{C}(\tilde{D}^{\uparrow{I}})$. Hence, $\epsilon_S(\mathcal{C}(\tilde{D}^{\uparrow{I}})) \coloneqq \mathcal{C}(\mathcal{C}(\tilde{D}^{\uparrow{I}}) \cap \gambles_S)= \mathcal{C}(\tilde{D}^{\uparrow{I}})$.
    \item $h(\mathcal{C}(\tilde{D}^{\uparrow{I}}),S)= (\tilde{D},S)$. In fact, $h(\mathcal{C}(\tilde{D}^{\uparrow{I}}),S) \coloneqq (( \epsilon_S(\mathcal{C}(\tilde{D}^{\uparrow{I}})) \cap \gambles_S )^{\downarrow{S}},S)= ( (\mathcal{C}(\tilde{D}^{\uparrow{I}})  \cap \gambles_S)^{\downarrow{S}},S)$ by previous item. Moreover, $( (\mathcal{C}(\tilde{D}^{\uparrow{I}})  \cap \gambles_S)^{\downarrow{S}},S)= (\tilde{D},S)$ by item 5 of Lemma~\ref{le:cylExtProperties}.
\end{itemize}
 So $h$ is surjective, hence bijective.
\end{enumerate}
\end{proof}

We remark that also in labeled information algebras, an information order can be defined analogously to the one seen for domain-free ones.


In a computational application of this second version of the labeled information algebra, one would use the fact that any set
$(\tilde{D},S)$ is determined by gambles defined on the set of possibilities $\Omega_S$, which reduce greatly the efficiency of storage. Observations like this explain why labeled information algebra are better suited for computational purposes.


\section{Atoms and Maximal Coherent Sets of Gambles} \label{sec:Atoms}

In certain information algebras there are maximally informative elements, called \textit{atoms}~\citep{kohlas03}. This is in particular the case for the information algebra of coherent sets of gambles. Maximal coherent sets of gambles $M$ (see Section~\ref{sec:DesGambles}) are clearly different from $\gambles$ and moreover have the property that, in information order,
\begin{eqnarray*}
M \leq D \textrm{ for}\ D \in \Phi \Rightarrow M = D \textrm{ or}\ D = \mathcal{L}(\Omega).
\end{eqnarray*}
Elements in an information algebra with these properties are called atoms~\citep{kohlas03}. In certain cases atoms determine fully the structure of an information algebra. We shall show that this is indeed the case for the algebra of coherent sets of gambles, see in particular the next section~\ref{sec:SetAlgs}. The definition of an atom can alternatively being expressed by combination. $M$ is an atom if 
\begin{eqnarray*}
M \cdot D = \mathcal{L}(\Omega) \textrm{ or}\ M \cdot D = M, \; \forall D \in \Phi.
\end{eqnarray*}
Let $At(\Omega)$ denote the set of all atoms (maximal coherent sets) on $\Omega$. For any set of gambles $D \in \Phi(\Omega)$, let $At(D)$ denote the subset of $At(\Omega)$ (maximal coherent sets) which contain $D$,
\begin{eqnarray*}
At(D) \coloneqq \{M \in At(\Omega):D \leq M\}.
\end{eqnarray*}
Clearly, this set depends on $D$ and $\Omega$. If $D \in \Phi(\Omega_S)$ for some $S \subseteq I$, then $At(D) \coloneqq \{M \in At(\Omega_S):D \leq M\}$. We will use this fact later on in this section.

In general such sets may be empty. Not so in the case of coherent sets of gambles. In the case of the information algebra of coherent sets of gambles, we have in fact a number of additional properties concerning atoms (see Section~\ref{sec:DesGambles}).
\begin{enumerate}
\item For any coherent set of gambles $D$, there is a maximal set (an atom) $M$ so that in information order  $D \leq M$ (i.e., $D \subseteq M$). So $At(D)$, for $D$ coherent, is never empty. An information algebra with this property is called \textit{atomic}.
\item For all coherent sets of gambles $D$, we have
\begin{eqnarray*}
D = \inf At(D) = \bigcap At(D).
\end{eqnarray*}
An information algebra with this property is called \textit{atomic composed}~\citep{kohlas03} or \textit{atomistic}.
\item For any, not empty, subset $A$ of $At(\Omega)$ we have that 
\begin{eqnarray*}
\inf A = \bigcap A
\end{eqnarray*}
is a coherent set of gambles, i.e., an element of $C(\Omega)$. Such an algebra is called \textit{completely atomistic}.
\end{enumerate}
The first two properties are proved in~\cite{CooQua12}, the third follows since coherent sets form a $\bigcap$-structure. Note that, if $A$ is a set of maximal sets of gambles, $A \subseteq At(\bigcap A)$, and in general $A$ is a proper subset of $At(\bigcap A)$. These properties determine the structure of the information algebra of coherent sets in term of so-called set algebras, as we shall discuss in the following section~\ref{sec:SetAlgs}. 

The labeled version $\tilde{\Psi}$ of the information algebra of coherent sets of gambles, exhibits some further structure of atoms. In fact, we may have maximally informative elements relative to a domain $\Omega_S$ for any set of variables $S \subseteq I$. If $\tilde{M} \in At(\Omega_S)$ and $\tilde{D} \in \Phi(\Omega_S)$, then
we have $( \tilde{D},S) \geq (\tilde{M},S)$ if and only if either $(\tilde{D},S) = (\tilde{M},S)$ or $(\tilde{D},S) = (\mathcal{L}(\Omega_S),S)$. This means that the elements $(\tilde{M},S)$ are maximally informative relative to the subset of variables $S$. Such elements are called \textit{atoms relative to $S$}~\citep{kohlas03}.
Such relative atoms have the following properties.

\begin{lemma} \label{le:PropOfAtoms}
Assume $(\tilde{M},S)$ and $(\tilde{M_1},S),(\tilde{M_2},S)$ to be atoms relative to $S$ and $(\tilde{D},S) \in \tilde{\Psi}$, with $S \subseteq I$. Then
\begin{enumerate}
\item $(\tilde{M},S) \cdot (\tilde{D},S) = (\mathcal{L}(\Omega_S),S)$ or $(\tilde{M},S) \cdot (\tilde{D},S)= (\tilde{M},S)$. 
\item If $T \subseteq S$, then $\pi_T(\tilde{M},S)$ is an atom relative to $T$.
\item Either $(\tilde{D},S) \leq (\tilde{M},S)$ or $(\tilde{D},S) \cdot (\tilde{M},S) = (\mathcal{L}(\Omega_S),S)$. 
\item Either $(\tilde{M}_{1},S) \cdot (\tilde{M}_{2},S) = (\mathcal{L}(\Omega_S),S)$ or $(\tilde{M}_{1},S) = (\tilde{M}_{2},S)$.
\end{enumerate}
\end{lemma}

These are purely properties of information algebras, for a proof see~\cite{kohlas03}. The properties of the algebra of being atomic, atomistic and completely atomistic carry over to the labeled version of the algebra of coherent sets. 
\begin{enumerate}
\item \textit{Atomic:} For any element $(\tilde{D},S)$ $\in \tilde{\Psi}_S, \; S \subseteq I$ with $\tilde{D} \in C(\Omega_S)$, there is an atom relative to $S$,  $(\tilde{M},S)$, so that $(\tilde{D},S) \leq (\tilde{M},S)$.
\item \textit{Atomistic:} For any element $(\tilde{D},S) \in \tilde{\Psi}_S, \; S \subseteq I$, with $\tilde{D} \in C(\Omega_S)$, $(\tilde{D},S) = \inf\{(\tilde{M} ,S):\tilde{M} \in At(\tilde{D})\}$.
\item \textit{Completly Atomistic:} For any, not empty, subset $A$ of $At(\Omega_S)$, $\inf\{(\tilde{M},S):\tilde{M} \in A\}$ exists and belongs to $\tilde{\Psi}_S$, for every $S \subseteq I$.
\end{enumerate}

As in the domain-free case these properties imply that the atoms determine the structure of the information algebra. This will be discussed in the next section.


\section{Set Algebras} \label{sec:SetAlgs}



Important instances of information algebras are \textit{set algebras}. We analyze here the set algebra whose elements are subsets of $\Omega$. The operation of combination is simply set intersection. Extraction is defined in terms of cylindrification: if $A$ is a subset of $\Omega$, then its cylindrification with respect to a subset $S$ of variables is defined as
\begin{eqnarray*}
\sigma_S(A) \coloneqq \{\omega \in \Omega: \exists \omega' \in A \textrm{ so that}\ \omega' \vert S = \omega \vert S\}.
\end{eqnarray*}
If $S= \emptyset$ then $\sigma_S(A)= \Omega$ for every $\emptyset \neq A \subseteq \Omega$. This is a saturation operator. The family of subsets of $\Omega$ with intersection as combination and cylindrification as extraction forms a domain-free information algebra~\citep{kohlas17}. Saturation operators are more generally defined relative to partitions or equivalence relations. In the present case we have, as before, the relations $\omega \equiv_S \omega'$ with $\omega, \omega' \in \Omega, \; S \subseteq I$, if $\omega \vert S = \omega' \vert S$. 

We show now that such a set algebra can be \emph{embedded} into the information algebra of coherent sets of gambles.\footnote{For the definition of embedding and in general of homomorphism between information algebras, see~\citealp{kohlas03}.} Define for any subset $A$ of $\Omega$
\begin{eqnarray*}
D_A \coloneqq \{f \in \mathcal{L}(\Omega): \inf_{\omega \in A} f(\omega) > 0\} \cup \mathcal{L}^+(\Omega).
\end{eqnarray*}
If $A$ is not empty, this is obviously a coherent set of gambles.
Now we have the following result.

\begin{theorem} \label{th:EmbedSetAlg}
For all subsets $A$ and $B$ of $\Omega$ and subsets $S$ of $I$
\begin{enumerate}
\item $D_\emptyset = \mathcal{L}(\Omega)$, $D_\Omega = \mathcal{L}^+(\Omega)$,
\item $D_A \cdot D_B = D_{A \cap B}$,
\item $\epsilon_S(D_A) = D_{\sigma_S(A)}$.
\end{enumerate}
\end{theorem} 

\begin{proof}
\begin{enumerate}
    \item They follow from the definition.
    \item Notice that $D_A=\mathcal{L}^+$ or $D_B = \mathcal{L}^+$ if and only if $A= \Omega$ or $B= \Omega$. Clearly in this case we have immediately the result.

The same is true if $D_A= \gambles$ or $D_B = \gambles$ that is equivalent to have $A= \emptyset$ or $B= \emptyset$.

Now, suppose $D_A,D_B \neq \mathcal{L}^+$ and $D_A,D_B \neq \mathcal{L}$.

If $A \cap B = \emptyset$, then $D_{A \cap B} = \mathcal{L}(\Omega)$. 
By definition we have $D_A \cdot D_B \coloneqq \mathcal{C}(D_A \cup D_B)$. Consider $f \in D_A$ and $g \in D_B$. Since $A$ and $B$ are disjoint, we have $\tilde{f} \in D_A$ and $\tilde{g} \in D_B$, where $\tilde{f}, \; \tilde{g}$ are defined in the following way:

\begin{eqnarray*}
\tilde{f}(\omega) \coloneqq \left\{ \begin{array}{ll} f(\omega) & \textrm{for}\ \omega \in A, \\ -g(\omega) & \textrm{for}\ \omega \in B, \\ 0 & \textrm{for}\ \omega \in  (A\cup B)^c\end{array} \right.
\\
\tilde{g}(\omega) \coloneqq \left\{ \begin{array}{ll} -f(\omega) & \textrm{for}\ \omega \in A , \\ g(\omega) & \textrm{for}\ \omega \in B, \\ 0 & \textrm{for}\ \omega \in (A\cup B)^c. \end{array} \right.
\end{eqnarray*}
However, $\tilde{f} + \tilde{g} = 0 \in \mathcal{E}(D_A \cup D_B)$, hence $D_A \cdot D_B = \mathcal{L}(\Omega) = D_{A \cap B}$.

Assume then that $A \cap B \not= \emptyset$. Note that $D_A \cup D_B \subseteq D_{A \cap B}$ so that $D_A \cdot D_B$ is coherent and $D_A \cdot D_B \subseteq D_{A \cap B}$. Consider then a gamble $f \in D_{A \cap B}$. 
Select a $\delta > 0$ and define the two functions
\begin{eqnarray*}
f_1(\omega) \coloneqq \left\{ \begin{array}{ll} 1/2f(\omega) & \textrm{for}\ \omega \in A \cap B, \omega \in  (A\cup B)^c, \\ \delta & \textrm{for}\ \omega \in A \setminus B, \\ f(\omega) - \delta & \textrm{for}\ \omega \in B \setminus A \end{array} \right.
\\
f_2(\omega) \coloneqq \left\{ \begin{array}{ll} 1/2f(\omega) & \textrm{for}\ \omega \in A \cap B, \omega \in  (A\cup B)^c, \\ f(\omega) - \delta & \textrm{for}\ \omega \in A \setminus B, \\ \delta & \textrm{for}\ \omega \in B \setminus A. \end{array} \right.
\end{eqnarray*}
Then, $f = f_1 + f_2$ and $f_1 \in D_A$, $f_2 \in D_B$. Therefore $f \in C(D_A \cup D_B) \eqqcolon D_A \cdot D_B$, hence $D_A \cdot D_B = D_{A \cap B}$. 

\item If $A$ is empty, then $\epsilon_S(D_\emptyset) = \mathcal{L}(\Omega)$ so that item 3 holds in this case. So assume $A \not= \emptyset$. Then we have
\begin{eqnarray*}
\epsilon_S(D_A) \coloneqq \mathcal{C}(D_A \cap \mathcal{L}_S) \coloneqq \posi(\mathcal{L}^+(\Omega) \cup (D_A \cap \mathcal{L}_S)). 
\end{eqnarray*}
Consider then a gamble $f \in D_A \cap \mathcal{L}_S$. Then, we have $\inf_A f > 0$ and $f$ is $S$-measurable. So, if $\omega \vert S = \omega' \vert S$ for some $\omega' \in A$ and $\omega \in \Omega$, then $f(\omega) = f(\omega')$. Therefore $\inf_{\sigma_S(A)} f = \inf_A f > 0$, hence $f \in D_{\sigma_S(A)}$. Then $\mathcal{C}(D_A \cap \mathcal{L}_S) \subseteq \mathcal{C}(D_{\sigma_S(A)})=D_{\sigma_S(A)}$.

Conversely, consider a gamble $f \in D_{\sigma_S(A)}$. $D_{\sigma_S(A)}$ is a strictly desirable set of gambles, hence, if $f \in D_{\sigma_S(A)}$, $f \in \mathcal{L}^+(\Omega)$ or there exists $\delta >0$ such that $f-\delta \in D_{\sigma_S(A)}$. If $f \in \mathcal{L}^+(\Omega)$, then $f \in \epsilon_S(D_A)$. Otherwise, let us define for every $\omega \in \Omega$,
\begin{eqnarray*}
g(\omega) \coloneqq \inf_{\omega' \vert S = \omega \vert S} f(\omega') - \delta.
\end{eqnarray*}

If $\omega \in A$, then $g(\omega) > 0$ since $\inf_{\sigma_S(A)} (f- \delta) > 0$. So we have $\inf_A g \ge 0$ and $g$ is $S$-measurable. However, $ \inf_A ( g + \delta) = \inf_A g + \delta >0$ hence $(g + \delta) \in D_A \cap \mathcal{L}_S$ and $f \geq g+ \delta$. Therefore
$f \in \mathcal{C}(D_A \cap \mathcal{L}_S)$.
\end{enumerate}
\end{proof}

This theorem shows that the map $A \mapsto D_A$ is an homomorphism between the set algebra and the information algebra of coherent sets of gambles. Furthermore, the map is one-to-one, hence it is an embedding of the set algebra into the algebra of coherent sets of gambles. This is a manifestation of the fact that the theory of desirable gambles covers among other things propositional and predicate logic.

Recall that $\Phi$ forms a lattice (Section~\ref{sec:DesGambles}) where meet is set intersection. This is also the case for subsets of $\Omega$; they form even a Boolean lattice. We need however to stress that in information order, $A \leq B$ iff $B \subseteq A$. That is, information order is the opposite of the usual inclusion order between sets. This means that meet is set union. Given this observation, it turns out that the homomorphism found is even a lattice homomorphism.

\begin{theorem}
For all subsets $A$ and $B$ of $\Omega$,
\begin{eqnarray*}
D_A \cap D_B = D_{A \cup B}.
\end{eqnarray*}
\end{theorem}

\begin{proof}
A gamble $f$ belongs to $D_A \cap D_B$ if and only if both $\inf_A f$ and $\inf_B f$ are both positive. But then it belongs to $D_{A \cup B}$.
\end{proof}

But there is much more about set algebras and information algebras of coherent sets of gambles. And this depends on the atomisticity of the information algebra of coherent sets of gambles. We stick in our discussion here to the domain-free view. The labeled view of what follows has been described in~\cite{kohlas03}. In the domain-free case the result we prove below states that $\Phi$ is homomorphic to a set algebra.

Consider the set of all atoms, that is all maximal (coherent) sets $At(\Phi)$, and define equivalence relations $M \equiv_S M'$ if $\epsilon_S(M) = \epsilon_S(M')$ in $At(\Phi)$ for all subsets of variables $S \subseteq I$. In what follows we limit our analysis to this kind of equivalence relations. Associated with them, there are saturation operators $\sigma_S$ defined by
\begin{eqnarray*}
\sigma_S(X) \coloneqq \{M \in At(\Phi):\exists M' \in X \textrm{ so that}\ M \equiv_S M'\},
\end{eqnarray*}
for any subset $X$ of $At(\Phi)$ and $S \subseteq I$. Any saturation operator satisfies a number of important properties which are related to information algebras.

\begin{lemma}
Let $\sigma_S$ be a saturation operator on $At(\Phi)$ for some $S \subseteq I$, associated with the equivalence relation $\equiv_S$, and $X,Y$ subsets of $At(\Phi)$. Then
\begin{enumerate}
\item $\sigma_S(\emptyset) = \emptyset$,
\item $X \subseteq \sigma_S(X)$,
\item $\sigma_S(\sigma_S(X) \cap Y) = \sigma_S(X) \cap \sigma_S(Y)$,
\item $X \subseteq Y$ implies $\sigma_S(X) \subseteq \sigma_S(Y)$,
\item $\sigma_S(\sigma_S(X)) = \sigma_S(X)$,
\item $X = \sigma_S(X)$ and $Y = \sigma_S(Y)$ imply $X \cap Y = \sigma_S(X \cap Y)$.
\end{enumerate}
\end{lemma}

\begin{proof}
Item 1., 2., 4. and 5. are obvious.

For 6. If $M \in \sigma_S(X \cap Y)$, then there is a $M' \in X \cap Y$ so that $M \equiv_S M'$. In particular, $M' \in X$, hence $M \in \sigma_S(X)$. At the same time, $M' \in Y$, hence $M \in \sigma_S(Y)$. Then $M \in \sigma_S(X) \cap \sigma_S(Y)= X \cap Y$. By 2. we must then have equality.

For 3. Observe that $\sigma_S(X) \cap Y \subseteq \sigma_S(X) \cap \sigma_S(Y)$, so that $\sigma_S(\sigma_S(X) \cap Y) \subseteq \sigma_S(\sigma_S(X) \cap \sigma_S(Y)) = \sigma_S(X) \cap \sigma_S(Y)$
by 4. and 6. For the reverse inclusion, note that $M \in \sigma_S(X) \cap \sigma_S(Y)$ means that there are $M' \in X$, $M'' \in Y$ so that $M \equiv_S M'$ and $M \equiv_S M''$. By transitivity we have then $M'' \equiv_S M'$ so that $M'' \in \sigma_S(X)$. Then $M \equiv_S M''$ and $M'' \in \sigma_S(X) \cap Y$ imply $M \in \sigma_S(\sigma_S(X) \cap Y)$. This concludes the proof.
\end{proof}

The first three items of the theorem correspond to the properties E1 to E3 of an existential quantifier in an information algebra, if combination is intersection. This is a first step to show that the subsets of $At(\Phi)$ indeed form an information algebra with intersection as combination and saturation operators $\sigma_S$ for $S \subseteq I$ as extraction operators. The missing item will be verified below.

First of all we need to define the combination of two equivalence relations.
\begin{definition}
Given two equivalence relations $\equiv_S$, $\equiv_T$ in $At(\Phi)$ with $S,T \subseteq I$, their combination is defined as
\begin{eqnarray*}
\equiv_S \cdot \equiv_T\ \coloneqq  \{(M,M') \in At(\Phi) \times At(\Phi):\exists M'' \in At(\Phi), \textrm{ so that}\ M \equiv_S M'' \equiv_T M'\}.
\end{eqnarray*}
\end{definition}
In general this is no more an equivalence relation. In our case however it is and the relations commute as the following lemma shows.

\begin{lemma} \label{le:ComEqRel}
For the equivalence relations $\equiv_S$ and $\equiv_T$ in $At(\Phi)$ with $S,T \subseteq I$, we have
\begin{eqnarray*}
\equiv_S \cdot \equiv_T\ =\ \equiv_T \cdot \equiv_S\ =\ \equiv_{S \cap T}.
\end{eqnarray*}
\end{lemma}

\begin{proof}
Let $(M,M') \in \equiv_S \cdot \equiv_T$ so that there is an $M''$ such that $\epsilon_S(M) = \epsilon_S(M'')$ and $\epsilon_T(M') = \epsilon_T(M'')$. It follows that $\epsilon_{S \cap T}(M) = \epsilon_T(\epsilon_S(M)) = \epsilon_T(\epsilon_S(M'')) = \epsilon_{S \cap T}(M'')$. Similarly we obtain $\epsilon_{S \cap T}(M') = \epsilon_{S \cap T}(M'')$. But this shows that $(M,M') \in\ \equiv_{S \cap T}$. 

Conversely, suppose $(M,M') \in \equiv_{S \cap T}$, that is $\epsilon_{S \cap T}(M) = \epsilon_{S \cap T}(M')$. We claim that $\epsilon_S(M) \cdot \epsilon_T(M') \not= 0$. If, on the contrary  $\epsilon_S(M) \cdot \epsilon_T(M') = 0$, then $\epsilon_S(\epsilon_S(M) \cdot \epsilon_T(M')) = \epsilon_S(M) \cdot \epsilon_{S \cap T}(M') = 0$ and further
$\epsilon_{S \cap T}(\epsilon_S(M) \cdot \epsilon_{S \cap T}(M')) = \epsilon_{S \cap T}(M) \cdot \epsilon_{S \cap T}(M') = \epsilon_{S \cap T}(M) = 0$. But since $M$ is an atom, this is a contradiction. So there is an atom $M'' \in At(\epsilon_S(M) \cdot \epsilon_T(M'))$ so that $\epsilon_S(M) \cdot \epsilon_T(M') \leq M''$. Then $\epsilon_S(M'') \geq \epsilon_S(M) \cdot \epsilon_{S \cap T}(M') = \epsilon_S(M)$. Therefore $\epsilon_S(M'') = \epsilon_S(M'') \cdot \epsilon_S(M) = \epsilon_S(\epsilon_S(M'') \cdot M) = \epsilon_S(M)$ since $M$ is an atom. In the same way, we deduce $\epsilon_T(M'') = \epsilon_T(M')$. But this means that $(M,M') \in\ \equiv_S \cdot \equiv_T$ and we have proved $\equiv_S \cdot \equiv_T\ =\ \equiv_{S \cap T}$. The other equality follows by symmetry.
\end{proof}
 
 As a corollary, it follows that the associated saturation operators commute. 
  
 \begin{lemma}
 For the equivalence relations $\equiv_S$ and $\equiv_T$ in $At(\Phi)$, $S,T \subseteq I$ we have
 \begin{eqnarray*}
\sigma_S \circ \sigma_T = \sigma_T \circ \sigma_S = \sigma_{S \cap T}.
\end{eqnarray*}
\end{lemma}

\begin{proof}
For any subset $X$ of $At(\Phi)$ we have
\begin{eqnarray*}
\lefteqn{\sigma_S \circ \sigma_T(X)  \coloneqq \{M \in At(\Phi):\exists M' \in X,\exists M'' \in At(\Phi) \textrm{ so that }M \equiv_S M'' \equiv_T M'\} } \\
&&= \{M \in At(\Phi):\exists M' \in X \textrm{ so tat } M \equiv_S \cdot \equiv_T M'\}  \\
&&= \{M \in At(\Phi):\exists M' \in X \textrm{ so that } M \equiv_{S \cap T} M'\} \eqqcolon \sigma_{S \cap T}(X).
\end{eqnarray*}
This proves that $\sigma_S \circ \sigma_T = \sigma_{S \cap T}$. The remaining equality follows by symmetry.
\end{proof}
 
By this result, we have established that $\mathcal{P}(At(\Phi))$ with intersection as combination and saturation as extraction satisfies all items of  Theorem~\ref{th:DomFreeInfAlg} (clearly we have also $\sigma_I(X) =X$ for every $X \subseteq At(\Phi)$). This means that it is a domain-free information algebra. Since its elements are subsets and combination and extraction are set operations, it is a \textit{set algebra}. But in addition $\mathcal{P}(At(\Phi))$ is also a complete, Boolean lattice under inclusion, which corresponds to information order in $\mathcal{P}(At(\Phi))$.

Now in the following theorem we show that 
$D \mapsto At(D)$ is an information algebra embedding and also maintain arbitrary joins. Recall that in information order in $\Phi$ we have $D_1 \leq D_2$ if and only if $D_1 \subseteq D_2$. So information order is inclusion. In $\mathcal{P}(At(\Phi))$ combination is intersection so the information order is $X \leq Y$ if $X \cap Y = Y$, hence $Y \subseteq X$. Therefore, information order in $\mathcal{P}(At(\Phi))$ is the inverse of inclusion, join is intersection in this order and meet union. Remark that the following theorems are purely results of atomistic information algebras and not specific to the algebra of coherent sets of gambles. Part of the following has been developed in~\cite{kohlasschmid16}.
\begin{theorem} \label{th:SertAlgInfom}
The map $D \mapsto At(D)$ that takes an element $D \in \Phi$ and maps it to $At(D) \in \mathcal{P}(At(\Phi))$, 
is injective. Moreover, if $D$, $D_1,D_2$ belong to $\Phi$ and $S$ is a subset of $I$, the following are valid
\begin{enumerate}
\item $At(\mathcal{L}^+(\Omega)) = At(\Phi)$ and $At(\mathcal{L}(\Omega)) = \emptyset$,
\item $At(D_1 \cdot D_2) = At(D_1) \cap At(D_2)$,
\item $At(\epsilon_S(D)) = \sigma_S(At(D))$.
\end{enumerate}
\end{theorem}

\begin{proof}
The map $D \mapsto At(D)$ is injective because $\Phi$, with the two operations of combination and extraction defined in Section~\ref{secDomFreeInfAlg}, is atomistic.

\begin{enumerate}
    \item By definition we have $At(\mathcal{L}(\Omega)) = \emptyset$ and since $\mathcal{L}^+(\Omega) \leq M$ for all atoms $M$, $At(\mathcal{L}^+(\Omega)) = At(\Phi)$ (recall $0 = \mathcal{L}(\Omega)$ and $1 = \mathcal{L}^+(\Omega)$). 
\item If there is an atom $M \in At(D_1 \cdot D_2)$, since $D_1,D_2 \leq D_1 \cdot D_2 \leq M$, we conclude that $M \in At(D_1) \cap At(D_2)$. On the other hand, if $M \in At(D_1) \cap At(D_2)$, then $D_1 \leq M$ and $D_2 \leq M$, hence $D_1 \cdot D_2 \coloneqq D_1 \vee D_2 \leq M$ and therefore $M \in At(D_1 \cdot D_2)$. 

\item If there is an atom $M \in \sigma_S(At(D))$, 
then there is a $M' \in At(D)$ so that $\epsilon_S(M) = \epsilon_S(M')$. Further $D \leq M'$ implies $\epsilon_S(D) \leq \epsilon_S(M') = \epsilon_S(M) \leq M$ so that 
 $M \in At(\epsilon_S(D))$.
 
 Conversely, if $M \in At(\epsilon_S(D))$ then $\epsilon_S(D) \leq M$. We have $D \leq \epsilon_S(M) \cdot D$. We claim that $\epsilon_S(M) \cdot D \not= 0$. Indeed, otherwise we would have $\epsilon_S(M \cdot \epsilon_S(D)) = \epsilon_S(M) \cdot \epsilon_S(D) = \epsilon_S(\epsilon_S(M) \cdot D) = \epsilon_S(0) = 0$ implying $M \cdot \epsilon_S(D) = 0$ which contradicts 
 $M \in At(\epsilon_S(D))$. So there exists an atom $M' \in At(\epsilon_S(M) \cdot D)$ and thus $D \leq \epsilon_S(M) \cdot D \leq M'$. We conclude that $M' \in At(D)$.

Further $\epsilon_S(\epsilon_S(M) \cdot D) = \epsilon_S(M) \cdot \epsilon_S(D) \leq \epsilon_S(M')$, hence $\epsilon_S(M) \cdot \epsilon_S(D) \cdot \epsilon_S(M') = \epsilon_S(M')$. This implies $\epsilon_S(M) \cdot \epsilon_S(M') \not= 0$. Since $\epsilon_S(M) \cdot \epsilon_S(M') = \epsilon_S(M \cdot \epsilon_S(M'))$, we conclude that $M \cdot \epsilon_S(M') \not= 0$, hence $\epsilon_S(M') \leq M$ since $M$ is an atom. It follows that $\epsilon_S(M') \leq \epsilon_S(M)$.

Proceed in the same way from $\epsilon_S(M) \cdot \epsilon_S(M') = \epsilon_S(\epsilon_S(M) \cdot M')$ in order to obtain $\epsilon_S(M) \leq \epsilon_S(M')$ so that finally $\epsilon_S(M) = \epsilon_S(M')$. But this means that $M \in \sigma_S(At(D))$. So we have proved that $At(\epsilon_S(D)) = \sigma_S(At(D))$.
\end{enumerate}
\end{proof}

This means that $\Phi$ is embedded in the set algebra $\mathcal{P}(At(\Phi))$ by the map $D \mapsto At(D)$. 
So, the algebra of coherent sets of gambles is essentially a set algebra in the technical sense used here. Note that this is purely a result of the theory of information algebra for atomistic algebras, and is not particular to the algebra of coherent sets of gambles. Recall however that we have already seen that $\Phi$ under information order is a complete lattice. And in fact the map $D \mapsto At(D)$ preserves arbitrary joins.

\begin{corollary}
Let $D_j$, $j \in J$ be an arbitrary family of sets of gambles such that $D_j \in \Phi$, for every $j \in J$. Then 
\begin{eqnarray*}
At(\bigvee_{j \in J} D_j) = \bigcap_{j \in J} At(D_j).
\end{eqnarray*}
\end{corollary}

\begin{proof}
The proof of item 2 of Theorem~\ref{th:SertAlgInfom} carries over to this more general case.
If there is an atom $M \in At(\bigvee_{j \in J} D_j)$, given that $D_j \subseteq \bigvee_{j \in J} D_j$ for every $j \in J$, we conclude that $M \in At(D_j)$ for every $j \in J$, hence $M \in \bigcap_{j \in J} At(D_j)$. Conversely, if $M \in \bigcap_{j \in J} At(D_j)$, then $D_j \subseteq M$ for all $j \in J$, hence $\bigvee_{j \in J} D_j \subseteq M$ and therefore $M \in At(\bigvee_{j \in J} D_j)$.
\end{proof}

We may ask how the images $At(D)$ of coherent sets of gambles are characterized in $At(\Phi)$. A (partial) answer is given in Section~\ref{sec:LinPrev}. We remark also that an analogous analysis can be made relative to the labeled view. We refer to~\cite{kohlas03} for this. In this view, the labeled algebras $\Psi$ or $\tilde{\Psi}$ are isomorphic to a generalized relational algebra in the sense of relational database theory~\citep{kohlas03}. We come back to this in Section~\ref{sec:LinPrev} with regard to lower previsions.


\section{Algebras of Lower and Upper Previsions} \label{sec:LowUpPrev}
A \emph{lower prevision} $\lpr$ is a function with values in $ \mathbb{R} \cup \{+\infty\}$ defined on some class of gambles $dom(\lpr)$, called the \emph{domain} of $\lpr$. Lower previsions, and their corresponding upper previsions, are the probabilistic models on which desirability is based~\citep{walley91}.\footnote{ Usually lower previsions are functions with values in $\mathbb{R}$. We consider here an extended version of this concept.}

We can also think of $\lpr(f)$ as the supremum buying price that a subject is willing to spend for the gamble $f$.
Following this interpretation, we can define it starting from the set of gambles $D$ that the subject is willing to accept~\citep{walley91, troffaes2014}. In this case, $dom(\lpr)$ is constituted by all the gambles $f$ for which $\{\mu \in \mathbb{R}:f - \mu \in D\}$ is not empty. 
\begin{definition}[Lower and upper prevision]
Given a non-empty set $D \subseteq \gambles(\Omega)$, we can associate to it the \emph{lower prevision} (operator) $\lpr: dom(\lpr) \rightarrow \mathbb{R} \cup \{+\infty\}$ defined as
\begin{eqnarray} \label{eq:LowPrevision}
\lpr(f) \coloneqq \sup\{\mu \in \mathbb{R}:f - \mu \in D\},
\end{eqnarray}
and the \emph{upper prevision} (operator) $\upr: dom(\upr) \rightarrow \mathbb{R} \cup \{-\infty\}$ defined as
\begin{eqnarray} \label{eq:UppPrevision}
\upr(f) \coloneqq -\lpr(-f),
\end{eqnarray}
where $dom(\lpr), dom(\upr)= -dom(\lpr) \subseteq \gambles(\Omega)$.
\end{definition}
Given the fact that we can always express upper previsions in terms of lower ones, in what follows we will concentrate only on lower previsions.
In the definition above we have not made explicit the dependence on $D$. However, when it is important to indicate it,  
we can see $\lpr$ also as the outcome of a function $\sigma$ applied to a set of gambles $D$ and write $\lpr = \sigma(D)$. We can also denote the set of gambles for which $\lpr$ is defined as $dom(\sigma(D))$. 



\begin{lemma}
Given a non-empty set of gambles $D \subseteq \mathcal{L}(\Omega)$, we have
\begin{enumerate}
\item $D \subseteq dom(\sigma(D))$.
\item If $0 \not \in \mathcal{E}(D)$, then $\sigma(D)(f) \in \mathbb{R}$ for every $f \in D$.
\item If $D \in C(\Omega)$, then $dom(\sigma(D))= \gambles(\Omega)$ and $\sigma(D)(f) \in \mathbb{R}$ for every $f \in \gambles(\Omega)$.
\item If $\overline{D} \in \overline{C}(\Omega)$, then $dom(\sigma(\overline{D}))= \gambles(\Omega)$ and $\sigma(\overline{D})(f) \in \mathbb{R}$ for every $f \in \gambles(\Omega)$.
\end{enumerate}
\end{lemma}

\begin{proof}
\begin{enumerate}
    \item Consider $f \in D$. Then the set $\{\mu \in \mathbb{R}:f - \mu \in D\}$ is not empty, since it contains at least $0$. 
    \item Assume $f - \mu \in D$. If $\mu \geq \sup f$, then $f(\omega) - \mu \le 0$ for all $\omega$, but then $0 \in \mathcal{E}(D)$, contrary to assumption. So, the set $\{\mu \in \mathbb{R}:f - \mu \in D\}$ is not empty and bounded from above for every $f \in D$.
    \item If $D$ is a coherent set of gambles, then $0 \not\in \mathcal{C}(D) = \mathcal{E}(D)=D$ so that $D \subseteq dom(\sigma(D))$ and $\sigma(D)(f) \in \mathbb{R}$, for every $f \in D$.
    Consider therefore $f \in \mathcal{L}(\Omega) \setminus D$. If there would be a $\mu \geq 0$ so that $f - \mu \in D$, then $f - \mu \leq f \in D$, which contradicts the assumption. Now, 
    if $\mu \leq \inf f < 0$, then $f - \mu \in \mathcal{L}^+(\Omega) \subseteq D$, so it follows $\inf f \leq \sigma(D)(f) < 0$ and $dom(\sigma(D)) = \mathcal{L}(\Omega)$. 
    \item From item 1., we have that $dom(\sigma(\overline{D})) \supseteq \overline{D}$. Moreover, if $\overline{D} \in \overline{C}(\Omega)$, then $-1 \notin \overline{D}$. Therefore, for every $f \in \overline{D}$, if $\mu \ge \sup f +1$ then $f- \mu \le -1 \notin \overline{D}$. So, $\{ \mu \in \mathbb{R}: \; f - \mu \in \overline{D} \}$ is bounded from above for every $f \in \overline{D}$.
    
    For $f \in \mathcal{L}(\Omega) \setminus \overline{D}$, we can repeat the procedure of item 3.
\end{enumerate}
\end{proof}

If $D$ is a coherent set of gambles, then the 
associated functional  $\lpr$ on $\mathcal{L}(\Omega)$ is called a coherent lower prevision. It is characterized by the following properties~\citep{walley91}. For every $f,g \in \mathcal{L}(\Omega)$:
\begin{enumerate}
\item $\lpr(f) \geq \inf_{\omega \in \Omega} f(\omega)$,
\item $\lpr(\lambda f) = \lambda \lpr(f)$, $\forall \lambda > 0$,
\item $\lpr(f + g) \geq \lpr(f) + \lpr(g)$.
\end{enumerate}

Let $\underline{\mathcal{P}}(\Omega)$ denote the family of coherent lower previsions. The map $\sigma$ maps $C(\Omega)$ to $\underline{\mathcal{P}}(\Omega)$. This map is not one-to-one as different coherent sets of gambles may induce the same lower prevision. We may apply the map $\sigma$ to almost desirable sets of gambles $\overline{D}$ and its range is still $\underline{\mathcal{P}}(\Omega)$ and we recall that moreover the map $\sigma$ restricted to almost desirable sets of gambles is one-to-one~\citep{walley91} and
\begin{eqnarray} \label{eq:AlmDesSetAndPrev}
\lpr(f) \coloneqq \max\{\mu  \in \mathbb{R}:f - \mu \in \overline{D}\}, \; \forall f \in \mathcal{L} , \quad \overline{D} \coloneqq \{f \in \gambles:\lpr(f) \geq 0\}.
\end{eqnarray}

There is also a one-to-one relation between coherent lower previsions $\lpr$ and strictly desirable sets of gamble $D^+$, so that, if we restrict $\sigma$ to strictly desirable sets, we have~\citep{walley91}
\begin{eqnarray*}
\lpr(f) \coloneqq \sup\{\mu \in \mathbb{R}:f - \mu \in D^+\} , \; \forall f \in \mathcal{L}, \quad D^+ \coloneqq \{f \in \gambles: \lpr(f) > 0\} \cup \mathcal{L}^+(\Omega).
\end{eqnarray*}

Define  the maps $\tau$ and $\overline{\tau}$ from coherent lower previsions to strictly desirable sets of gambles and almost desirable sets of gambles accordingly by
\begin{eqnarray*}
\tau(\lpr) \coloneqq \{f \in \gambles: \lpr(f) > 0\} \cup \mathcal{L}^+(\Omega), \quad \overline{\tau}(\lpr) \coloneqq  \{f \in \gambles:\lpr(f) \geq 0\}.
\end{eqnarray*}
Then $\tau$ and $\overline{\tau}$ are the inverses of the map $\sigma$ restricted to strictly desirable and almost desirable sets of gambles respectively. The following lemma shows how coherent, strictly desirable and almost desirable sets are linked relative to the coherent lower previsions they induce.\footnote{
This result follows also from the fact that, in the sup-norm topology, given a coherent set $D$ its relative interior plus the non-negative, non-zero gambles $D^+$ is a strictly desirable set of gambles and $\overline{D}$, the relative closure of $D$, is an almost desirable set of gambles~\citep{walley91}.}

\begin{lemma}
Let $D$ be a coherent set of gambles. Then
\begin{eqnarray*}
D^+ \coloneqq \tau(\sigma(D)) \subseteq D \subseteq \overline{\tau}(\sigma(D)) \coloneqq \overline{D}
\end{eqnarray*}
and $\sigma(D^+) = \sigma(D) = \sigma(\overline{D})$.
\end{lemma}

\begin{proof}
Let $\lpr \coloneqq \sigma(D)$. Then $f \in D^+$ means that $0 < \lpr(f) \coloneqq \sup\{\mu \in \mathbb{R}:f - \mu \in D\}$, or $f \in \mathcal{L}^+(\Omega)$. If $f \in  \mathcal{L}^+(\Omega)$ then $f \in D$. Otherwise, there is a $\delta$ such that $0 < \delta < \lpr(f)$ and $f-\delta \in D$. Therefore $f \in D$ and $D^+ \subseteq D$. Further, consider $f \in D$. Then we must have $\lpr(f) \coloneqq \sup\{\mu \in \mathbb{R}:f - \mu \in D\} \geq 0$, hence $f \in \overline{D}$. The second part follows since $\tau$ and $\overline{\tau}$ are the inverse maps of $\sigma$ on strictly desirable and almost desirable sets of gambles.
\end{proof}

Recall that strictly desirable sets of gambles form a subalgebra of the information algebra of coherent sets of gambles. Then this result establishes a map $D \mapsto D^+$ for any coherent set of gambles to a strictly desirable set.  We shall see that this map preserves combination and extraction. To prove this theorem we need the following lemma.

\begin{lemma} \label{le:IntOfCohSet}
Let $D$ be a coherent set of gambles and $D^+ \coloneqq \tau(\sigma(D))$, if $f \notin \mathcal{L}^+(\Omega)$ then $f \in D^+$ if and only if there is a $\delta > 0$ so that $f - \delta \in D$.
\end{lemma}

\begin{proof}
One part is by definition: if $f \in D^+$ and $f \notin \mathcal{L}^+(\Omega)$, then there is a $\delta > 0$ so that $f - \delta \in D^+ \subseteq D$. Conversely, consider $f  - \delta \in D$ for some $\delta > 0$ with $f \notin \mathcal{L}^+(\Omega)$ and note that $D^+ \coloneqq \{g \in \gambles:\sigma(D)(g) > 0\} \cup \mathcal{L}^+(\Omega)$. We have $\sigma(D)(g) \coloneqq \sup\{\mu \in \mathbb{R}:g- \mu \in D\}$ for every $g \in \gambles$. From $f  - \delta \in D$ we deduce that $\sigma(D)(f) > 0$, hence $f \in D^+$. 
\end{proof}

The next theorem establishes that this map is a weak homomorphism, weak, because it does not apply when $D_1$ and $D_2$ are mutually inconsistent, that is if $D_1 \cdot D_2 = 0$ (see Example 1 below).

Before stating this result we need to introduce a partial order relation on lower previsions.
Indeed, we define $\lpr \leq \underline{Q}$ and we say that $\underline{Q}$ \emph{dominates} $\lpr$, if
$dom(\lpr) \subseteq dom(\underline{Q})$ and $\lpr(f) \leq \underline{Q}(f)$ for all $f \in dom(\lpr)$. This is a partial order on lower previsions~\citep{troffaes2014}.
Note that, in particular, if $D' \subseteq D$ such that $0 \notin \mathcal{E}(D')$ and $D$ coherent, then $\sigma(D') \le \sigma(D)$. So, in particular, $\sigma(D') \le \sigma(\edesirs(D'))$ coherent. Vice versa, if instead we consider only coherent lower previsions, then also $\tau, \overline{\tau}$ preserve order.

\begin{theorem} \label{th:CohtoSXtrict}
Let $D_1$, $D_2$ and $D$ be coherent sets of gambles and $S \subseteq I$. 
\begin{enumerate}
\item If $D_1 \cdot D_2 \not= 0$, then $D_1 \cdot D_2 \mapsto (D_1 \cdot D_2)^+ = D_1^+ \cdot D_2^+$,
\item $\epsilon_S(D) \mapsto (\epsilon_S(D))^+ = \epsilon_S(D^+)$.
\end{enumerate}
\end{theorem}

\begin{proof}
\begin{enumerate}
    \item Note first that $D_1^+ \subseteq D_1$ and $D_2^+ \subseteq D_2$ so that
\begin{eqnarray*}
D_1^+ \cdot D_2^+ = \tau(\sigma(D_1^+ \cdot D_2^+)) \subseteq \tau(\sigma(D_1 \cdot D_2)) \eqqcolon (D_1 \cdot D_2)^+.
\end{eqnarray*}

Further
\begin{eqnarray*}
(D_1 \cdot D_2)^+ \coloneqq \tau(\sigma(D_1 \cdot D_2)) \coloneqq \{f \in \gambles:\sigma(D_1 \cdot D_2)(f) > 0\} \cup \mathcal{L}^+(\Omega).
\end{eqnarray*}
So, if $f \in (D_1 \cdot D_2)^+$, then either $f \in\mathcal{L}^+(\Omega)$ or 
\begin{eqnarray}
\sigma(D_1 \cdot D_2)(f) \coloneqq \sup\{\mu \in \mathbb{R}:f - \mu \in \mathcal{C}(D_1 \cup D_2)\} > 0.
\end{eqnarray}
In the first case obviously $f \in D_1^+ \cdot D_2^+$. Let us consider now $f \notin \gambles^+$, in this case there is a $\delta > 0$ so that $f -\delta \in \mathcal{C}(D_1 \cup D_2)$. This means that 
$f - \delta = h + \lambda_1 f_1 + \lambda_2 f_2$, where $h \in \mathcal{L}^+(\Omega) \cup \{0\}$, $f_1 \in D_1$, $f_2 \in D_2$ and $\lambda_1, \lambda_2 \geq 0$ and not both equal $0$. But then 
\begin{eqnarray*}
f = h + (\lambda_1 f_1 + \delta/2) + (\lambda_2 f_2 + \delta/2).
\end{eqnarray*}
We have $f'_1 \coloneqq \lambda_1 f_1 + \delta/2 \in D_1$ and $f'_2 \coloneqq \lambda_2 f_2 + \delta/2 \in D_2$. But this, together with $\lambda_1f_1 = f'_1 - \delta/2 \in D_1$ if $\lambda_1>0$ or otherwise $f'_1 \in \mathcal{L}^+(\Omega)$, and $\lambda_2f_2 = f'_2 - \delta/2 \in D_2$ if $\lambda_2>0$ or otherwise $f'_2 \in \mathcal{L}^+(\Omega)$, show according to Lemma~\ref{le:IntOfCohSet} that $f'_1 \in D_1^+$ and $f'_2 \in D_2^+$. So, finally, we have $f \in D_1^+ \cdot D_2^+ = \mathcal{C}(D_1^+ \cup D_2^+)$. This proves that $(D_1 \cdot D_2)^+ = D_1^+ \cdot D_2^+$.
\item Note that $D^+ \subseteq D$ so that
\begin{eqnarray*}
\epsilon_S(D^+)= \tau(\sigma(\epsilon_S(D^+))) \subseteq \tau(\sigma(\epsilon_S(D)))\eqqcolon (\epsilon_S(D))^+.
\end{eqnarray*}
This is valid, because $\epsilon_S(D^+)$ is also strictly desirable. Further
\begin{eqnarray*}
(\epsilon_S(D))^+ \coloneqq \tau(\sigma(\epsilon_S(D))) \coloneqq \{f \in \gambles:\sigma(\epsilon_S(D)) > 0\} \cup \mathcal{L}^+(\Omega),
\end{eqnarray*}
where
\begin{eqnarray*}
\sigma(\epsilon_S(D))(f) \coloneqq \sup \{\mu \in \mathbb{R}:f - \mu \in \mathcal{C}(D \cap \mathcal{L}_S)\}.
\end{eqnarray*}
So, if $f \in (\epsilon_S(D))^+$, then either $f \in \mathcal{L}^+(\Omega)$ in which case $f \in  \epsilon_S(D^+)$ or there is a $\delta > 0$ so that $f - \delta \in \mathcal{C}(D \cap \mathcal{L}_S) = \posi\{\mathcal{L}^+(\Omega) \cup (D \cap \mathcal{L}_S)\}$. In the second case, if $f \notin \gambles^+$, $f - \delta = h + g$ where $h \in \mathcal{L}^+(\Omega) \cup \{0\}$ and $g \in D \cap \mathcal{L}_S$. Then we have $f = h + g'$ where $g' \coloneqq g + \delta$ is still $S$- measurable and $g' \in D$. 
But, given the fact that $g = g' - \delta \in D \cap  \mathcal{L}_S$, from Lemma~\ref{le:IntOfCohSet}, we have $g' \in D^+ \cap \mathcal{L}_S$ and therefore $f \in \epsilon_S(D^+)$. Thus we conclude that $(\epsilon_S(D))^+ = \epsilon_S(D^+)$ which concludes the proof.
\end{enumerate}
\end{proof}

 Next, we claim that the map $\sigma$ restricted to coherent sets of gambles preserves also infima. Here we define the functional $\inf\{\lpr_j:j \in J\}$ by $\inf\{\lpr_j(f):j \in J\}$ for all $f \in \mathcal{L}(\Omega)$.

\begin{lemma} \label{le:MaintainInf}
Let $D_j$, $j \in J$ be any family of coherent sets. Then we have
\begin{eqnarray*}
\sigma(\bigcap_{j \in J} D_j) = \inf\{\sigma(D_j) :j \in J\}.
\end{eqnarray*}
\end{lemma}

\begin{proof}
Note that the intersection of the coherent sets $D_j$ equals a coherent set $D$. We have $\sigma(\bigcap_{j \in J} D_j) \coloneqq \sigma(D) \eqqcolon \lpr \leq \sigma(D_j)$, for all $j \in J$. So $\lpr$ is a lower bound of $\sigma(D_j), j \in J$, therefore $\lpr \le \inf\{\sigma(D_j) :j \in J\}$. However, $\inf\{\sigma(D_j): j \in J\}$ is coherent~\citep{walley91}. 
Then we have $\tau(\inf\{\sigma(D_j): j \in J\}) \subseteq \tau(\sigma(D_j)) \subseteq D_j$ for all $j \in J$, by definition of $\inf\{\sigma(D_j): j \in J\}$. 
Hence $\tau(\inf\{\sigma(D_j): j \in J\}) \subseteq \bigcap_j D_j \eqqcolon D$. But this implies $\inf\{\sigma(D_j): j \in J\} \leq \sigma(D) \eqqcolon \lpr$. Thus $\lpr$ is indeed the infimum of the $\sigma(D_j)$ for $j \in J$.
\end{proof}

If $\lpr$ is a lower prevision which is dominated by a coherent lower prevision, then its natural extension is defined as the infimum of the coherent lower previsions which dominate it~\citep{walley91},
\begin{eqnarray} \label{eq:NatExtLowPrev}
\underline{E}(\lpr) \coloneqq \inf\{\lpr' \in \underline{\mathcal{P}}(\Omega): \lpr \leq \lpr'\}.
\end{eqnarray}
So, $\underline{E}(\lpr)$ is the minimal coherent lower prevision which dominates $\lpr$. As usual, the operator applied to a lower prevision $\lpr$, changes with the possibility set on which are defined those gambles over that $\lpr$ operates.
We can also introduce another operator defined for every lower prevision $\lpr$ as follows:
\begin{eqnarray*}
\underline{E}^*(\lpr)(f) \coloneqq
\begin{cases}
   \underline{E}(\lpr)(f) &\text{ if } \exists \lpr' \in \underline{\mathcal{P}}:  \lpr \le \lpr', \\
    \infty &\text{ otherwise}, \notag
  \end{cases}
\end{eqnarray*}
for every $f \in \gambles$, which clearly changes too with the possibility set on which are defined those gambles over that $\lpr$ operates.

Now we prove the key result: the map $\sigma$ commutes with
natural extension under certain conditions.

\begin{theorem} \label{th:CommSigmaExt}
Let $D'$ be a non-empty set of gambles which satisfies the following two conditions
\begin{enumerate}
\item $0 \not\in \mathcal{E}(D')$,
\item for all $f \in D' \setminus \mathcal{L}^+(\Omega)$ there exists a $\delta > 0$ such that $f - \delta \in D'$.
\end{enumerate}
Then we have
\begin{eqnarray*}
\sigma(\mathcal{C}(D')) = \underline{E}^*(\sigma(D'))= \underline{E}(\sigma(D')).
\end{eqnarray*}
\end{theorem}

\begin{proof}
If $D' = \mathcal{L}^+(\Omega)$, then $D' \in C(\Omega)$ and $\sigma(\mathcal{C}(D')) = \underline{E}(\sigma(D'))$ because the lower prevision associated with $\mathcal{L}^+(\Omega)$ is already coherent. So, assume that $D' \not= \mathcal{L}^+(\Omega)$. We have then $\mathcal{E}(D') = \mathcal{C}(D') \in C(\Omega)$, so that
\begin{eqnarray*}
\mathcal{C}(D') = \bigcap \{D:D \textrm{ coherent},D' \subseteq D\}.
\end{eqnarray*}
Let $\lpr' \coloneqq \sigma(D')$, then  $\sigma(\mathcal{C}(D')) \ge \lpr'$ and moreover $\sigma(\mathcal{C}(D'))$ is coherent, hence $\sigma(\mathcal{C}(D')) \ge \underline{E}(\lpr')$, where $\underline{E}(\lpr')$ defined by Eq.~\eqref{eq:NatExtLowPrev}. 

Now, $\mathcal{E}(D')$ is a strictly desirable set of gambles such that $\mathcal{E}(D') \supseteq D'$. So, there exists at least a strictly desirable set of gambles containing $D'$. Therefore we have:
\begin{eqnarray*}
\sigma(\mathcal{C}(D')) =
 \sigma(\{\bigcap\{D:D \textrm{ coherent},D' \subseteq D\})
\leq   \sigma(\bigcap \{D^+ \textrm{ strictly desirable}:D' \subseteq D^+\}).
\end{eqnarray*}
Clearly, if $D' \subseteq D^+$, then $\lpr' \le \lpr \coloneqq \sigma(D^+)$, where $\lpr$ is a coherent lower prevision.
We claim that the converse is also valid. Indeed, let us consider a coherent lower prevision $\lpr$ such that  $\lpr' \leq \lpr$ and its associated strictly desirable set of gambles $D^+ \coloneqq \tau(\lpr)$. 
If $f \in D'$, then $\lpr'(f) \geq 0$. If $f \in \mathcal{L}^+(\Omega)$, then $f \in \tau(\lpr)$, otherwise, if $f \in D' \setminus \mathcal{L}^+(\Omega)$, then there is by assumption a $\delta > 0$ such that $f - \delta \in D'$, hence $0 < \lpr'(f) \leq \lpr(f)$. But this means again that $f \in \tau(\lpr)$. 
So, thanks to Lemma~\ref{le:MaintainInf} we have:
\begin{eqnarray*}
\sigma(\bigcap \{D^+ \textrm{ strictly desirable} :D' \subseteq D^+\})=\inf \{\lpr \textrm{ coherent}:\lpr' \leq \lpr\} \eqqcolon \underline{E}(\lpr')
 \end{eqnarray*}
so that $\sigma(\mathcal{C}(D')) = \underline{E}(\sigma(D'))$,
concluding the proof.
\end{proof}

From this result on natural extensions in the two formalism of coherent sets of gambles and coherent lower previsions, we can now introduce into $\underline{\Phi}(\Omega) \coloneqq  \underline{\mathcal{P}}(\Omega) \cup \{ \sigma(\mathcal{L}(\Omega))\}$, where $\sigma(\mathcal{L}(\Omega))(f) \coloneqq \infty$ for all $f \in \mathcal{L}(\Omega)$, like in $\Phi(\Omega)$, operations of combination and extraction. As usual, when there is no possible ambiguity we can also indicate $\underline{\Phi}(\Omega)$ simply with $\underline{\Phi}$. As in Section~\ref{secDomFreeInfAlg} then, consider a family of variables $X_i$, $i \in I$ with domains $\Omega_i$ and subsets $S \subseteq I$ of variables with domains $\Omega_S$. Let then for two coherent sets of gambles which are not inconsistent, $D_1 \cdot D_2 \not = 0$, with $\lpr_1 \coloneqq \sigma(D_1)$ and $\lpr_2 \coloneqq \sigma(D_2)$,
\begin{eqnarray*}
\lpr'(f) \coloneqq \sigma(D_1 \cup D_2)(f) \coloneqq \sup \{\mu \in \mathbb{R}:f - \mu \in D_1 \cup D_2\} = \max\{\lpr_1(f),\lpr_2(f)\}
\end{eqnarray*}
or $\lpr' = \max\{\lpr_1,\lpr_2\}$. We may take $\underline{E}^*(\lpr')$ to define combination of two lower previsions $\lpr_1$ and $\lpr_2$ in $\underline{\Phi}$. For extraction, for every $S \subseteq I$, we may take the $\underline{E}^*(\lpr_S)$, where of $\lpr_S$ is defined as the restriction of $\lpr \in \underline{\Phi}$ to $\mathcal{L}_S$. Thus, in summary, we can define on  $\underline{\Phi}$ the following operations.
\begin{enumerate}
\item Combination:
\begin{eqnarray*}
\lpr_1 \cdot \lpr_2\coloneqq \underline{E}^*(\max\{\lpr_1,\lpr_2\}).
\end{eqnarray*}
\item Extraction:
\begin{eqnarray*}
\underline{e}_S(\lpr) \coloneqq \underline{E}^*(\lpr_S).
\end{eqnarray*}
\end{enumerate}
The following theorem permits to conclude that $\underline{\Phi}$ with the two operations defined above, forms a domain-free information algebra. With the same little abuse of notation introduced before for coherent sets of gambles, we can call it the \emph{domain-free information algebra of coherent lower previsions}. 

\begin{theorem} \label{th:IsomPrev} \ 
Let $D_1^+$, $D_2^+$ and $D^+$ be strictly desirable sets of gambles and $S \subseteq I$. Then
\begin{enumerate}
\item $\sigma(\mathcal{L}(\Omega))(f) = \infty$, $\sigma(\mathcal{L}^+(\Omega))(f) = \inf f$ for all $f \in \mathcal{L}(\Omega)$,
\item $\sigma(D_1^+ \cdot D_2^+) = \sigma(D_1^+) \cdot \sigma(D_2^+)$,
\item $\sigma(\epsilon_S(D^+)) = \underline{e}_S(\sigma(D^+))$.
\end{enumerate}
\end{theorem}

\begin{proof}
\begin{enumerate}
    \item It follows from the definition.
    \item Assume first that $D_1^+ \cdot D_2^+ = 0$ and let $\lpr_1 \coloneqq \sigma(D_1^+)$, $\lpr_2 \coloneqq \sigma(D_2^+)$. Then there can be no coherent lower prevision $\lpr$ dominating both $\lpr_1$ and $\lpr_2$. Indeed, otherwise we would have $D_1^+ = \tau(\lpr_1) \leq \tau(\lpr)$ and $D_2^+ = \tau(\lpr_2) \leq \tau(\lpr)$, where $\tau(\lpr)$ is a coherent set of gambles. But this is a contradiction. Therefore, we have $\sigma(D_1^+ \cdot D_2^+)(f) = \infty = (\sigma(D_1^+) \cdot \sigma(D_2^+))(f)$, for all gambles $f \in \gambles$.

Let then $D_1^+ \cdot D_2^+ \not= 0$. Then $D_1^+ \cdot D_2^+$ as well as $D_1^+ \cup D_2^+$ satisfy the condition of Theorem~\ref{th:CommSigmaExt}. Therefore, applying this theorem, we have
\begin{eqnarray*}
\lefteqn{\sigma(D_1^+ \cdot D_2^+) \coloneqq \sigma(\mathcal{C}(D_1^+ \cup D_2^+)) = \underline{E}(\sigma(D_1^+ \cup D_2^+)) }\\
&&= \underline{E}(\max\{\sigma(D_1^+),\sigma(D_2^+)\}) \eqqcolon \sigma(D_1^+) \cdot \sigma(D_2^+).
\end{eqnarray*}

\item We remark that $D^+ \cap \mathcal{L}_S$ satisfies the condition of Theorem~\ref{th:CommSigmaExt}. Thus we obtain
\begin{eqnarray*}
\sigma(\epsilon_S(D^+)) \coloneqq \sigma(\mathcal{C}(D^+ \cap \mathcal{L}_S)) = \underline{E}(\sigma(D^+ \cap \mathcal{L}_S)).
\end{eqnarray*}
Now, 
\begin{eqnarray*}
\sigma(D^+ \cap \mathcal{L}_S)(f) \coloneqq \sup\{\mu \in \mathbb{R}:f - \mu  \in D^+ \cap \mathcal{L}_S\}, \; \forall f \in \gambles.
\end{eqnarray*}
But $f - \mu \in D^+ \cap \mathcal{L}_S$ if and only if $f$ is $S$-measurable and $f - \mu \in D^+$. Therefore, we conclude that $\sigma(D^+ \cap \mathcal{L}_S) = \sigma(D^+)_S$. Thus, we have indeed $\sigma(\epsilon_S(D^+)) = \underline{E}(\sigma(D^+)_S) \eqqcolon \underline{e}_S(\sigma(D^+))$. 
\end{enumerate}
\end{proof}

If we define $\tau(\sigma(\gambles(\Omega))) \coloneqq \gambles(\Omega)$, then the map $\sigma$, restricted to the information algebra $\Phi^+$, is bijective. It follows that the set $\underline{\Phi} $ is a domain-free information algebra, isomorphic to $\Phi^+$. There is obviously the connected (isomorphic) information algebra of upper previsions. The following corollary, shows furthermore that $\Phi$ is weakly homomorphic to $\underline{\Phi}$.

\begin{corollary} \label{corr:HomLowPrev}
Let $D_1$, $D_2$ and $D$ be coherent sets of gambles so that $D_1 \cdot D_2 \not= 0$ and $S \subseteq I$. Then
\begin{enumerate}
\item $\sigma(D_1 \cdot D_2) = \sigma(D_1) \cdot \sigma(D_2)$,
\item $\sigma(\epsilon_S(D)) = \underline{e}_S(\sigma(D))$.
\end{enumerate}
\end{corollary}

\begin{proof}
These claims are immediate consequences of Theorems~\ref{th:CohtoSXtrict} and Theorem~\ref{th:IsomPrev}. 
\end{proof}

The homomorphism does not extend to a pair of inconsistent coherent sets of gambles, as the following example shows.

\begin{example}\label{1.}
Consider a set of possibilities $\Omega \coloneqq \{\omega_1,\omega_2\}$ and let 
\begin{eqnarray*}
D_1 &\coloneqq& \{f \in \gambles:f(\omega_2) > -2 f(\omega_1)\} \cup \{f :f(\omega_2) = -2 f(\omega_1),f(\omega_2) \geq 0,f \not= 0\}, \\
D_2 &\coloneqq& \{f \in \gambles:f(\omega_2) > -2 f(\omega_1)\} \cup \{f \in \gambles:f(\omega_2) = -2 f(\omega_1),f(\omega_2) < 0\}.
\end{eqnarray*}
These are coherent sets of gambles, but they are mutually inconsistent, since we have $D_1 \cdot D_2 = \mathcal{L}(\Omega)$ because $0 \in \mathcal{E}(D_1 \cup D_2)$. But on the other hand,
\begin{eqnarray*}
D_1^+ = D_2^+ = \{f \in \gambles :f(\omega_2) > -2 f(\omega_1)\}
\end{eqnarray*}
therefore $D_1 \cdot D_2 =0$ while $D_1^+ \cdot D_2^+=D_1^+$ and $\sigma(D_1) \cdot \sigma(D_2)= \sigma(D_1^+) \cdot \sigma(D_2^+)= \sigma(D_1^+) \not= \sigma(\mathcal{L}(\Omega))$. So we have also 
$\sigma(D_1 \cdot D_2) \not= \sigma(D_1) \cdot \sigma(D_2)$ in this example.
\end{example}

Finally, it follows from Lemma~\ref{le:MaintainInf} and Corollary~\ref{corr:HomLowPrev} that for any family of coherent sets of gambles $D_j$,
\begin{eqnarray*}
\sigma(\epsilon_S(\bigcap_j D_j)) = \underline{e}_S(\sigma(\bigcap_j D_j)) = \underline{e}_S(\inf\{\sigma(D_j)\})
\end{eqnarray*}
and
\begin{eqnarray*}
\sigma(\bigcap_j \epsilon_S(D_j)) = \inf\{\sigma(\epsilon_S(D_j))\} = \inf\{\underline{e}_S(\sigma(D_j))\}.
\end{eqnarray*}
Therefore it also follows from Eq.~\eqref{eq:ExzrMeet} that for any family $\lpr_j$ of coherent lower previsions, we have 
\begin{eqnarray*}
\underline{e}_S(\inf\{\lpr_j\}) = \inf\{\underline{e}_S(\lpr_j)\}.
\end{eqnarray*}
So in the information algebra of coherent lower previsions extraction distributes still over meet (infimum).

As always there is also a labeled version of the domain-free information algebra of coherent lower previsions. 

Let us introduce the concept of support also for lower previsions.
\begin{definition}[Support for lower previsions] A subset $S$ of $I$ is called \emph{support} of a lower prevision $\lpr \in \underline{\Phi}$, if $\underline{e}_S(\lpr) =\lpr$.
\end{definition}
Now, we can construct the first version of the labeled information algebra of coherent lower previsions, using the same standard procedure described in Section~\ref{sec:LabInfAlg} for coherent set of gambles. So, let us define $\underline{\Psi}_S(\Omega) \coloneqq \{(\lpr,S): S \text{ is a support of } \lpr \in \underline{\Phi}(\Omega) \}$ for every $S \subseteq I$ and $\underline{\Psi}(\Omega) \coloneqq \bigcup_{S \subseteq I} \underline{\Psi}_S(\Omega)$. It is possible to define on $\underline{\Psi}(\Omega)$, analogously to the first version of the labeled information algebra of coherent set of gambles, the following operations:
\begin{enumerate}
\item Labeling: $d(\lpr,S) \coloneqq S$.
\item Combination: $(\lpr_1,S) \cdot (\lpr_2,T) \coloneqq (\lpr_1 \cdot \lpr_2,S \cup T)$.
\item Projection (Marginalization): $\underline{\pi}_T(\lpr,S) \coloneqq (\underline{e}_T(\lpr),T)$, for every $T \subseteq S \subseteq I$.
\end{enumerate}
Moreover, with the same procedure used in Section~\ref{sec:LabInfAlg}, we can show that $\underline{\Psi}(\Omega)$ with the operations above forms a labeled information algebra. This is in fact, a standard method for constructing a labeled information algebra from a domain-free one and it does not depend on the specific domain-free information algebra considered.

Now, we can proceed constructing the second version of this labeled algebra, again following the same reasoning used for coherent sets of gambles.
So, for any subset $S$ of $I$ let
\begin{eqnarray*}
\underline{\tilde{\Psi}}_S(\Omega) \coloneqq \{(\tilde{\lpr},S): \tilde{\lpr} \in \underline{\Phi}(\Omega_S) \}.
\end{eqnarray*}

Further let
\begin{eqnarray*}
\underline{\tilde{\Psi}}(\Omega) \coloneqq \bigcup_{S \subseteq I} \underline{\tilde{\Psi}}_S(\Omega).
\end{eqnarray*}
As usual, we can refer to $\underline{\tilde{\Psi}}_S(\Omega)$ and $\underline{\tilde{\Psi}}(\Omega)$, also with $\underline{\tilde{\Psi}}_S$ and $\underline{\tilde{\Psi}}$ respectively, when there is no possible ambiguity.

Within $\underline{\tilde{\Psi}}$ we define the following operations.
\begin{enumerate}
\item Labeling: $d(\tilde{\lpr},S) \coloneqq S$.
 
\item Combination:
\begin{eqnarray*}
(\tilde{\lpr}_1,S) \cdot (\tilde{\lpr}_2,T) \coloneqq (\underline{E}^*( \tilde{\lpr}_1^{\uparrow{S \cup T}})  \cdot \underline{E}^*( \tilde{\lpr}_2^{\uparrow{S \cup T}}) , S \cup T),
\end{eqnarray*}
where, given a lower prevision $\lpr$ with $dom(\lpr) \subseteq \gambles(\Omega_Z)$ with $Z \subseteq S \cup T$,  $\lpr^{\uparrow{S \cup T}}(f) \coloneqq \lpr(f^{\downarrow{Z}})$, for every $f \in \gambles_Z(\Omega_{S \cup T})$ such that $f^{\downarrow{Z}} \in dom(\lpr)$.


\item Projection (Marginalization):
\begin{eqnarray*}
\underline{\pi}_T(\tilde{\lpr},S)\coloneqq (\underline{e}_T(\tilde{\lpr} )_T^{\downarrow{T}},T),
\end{eqnarray*}
for every $T \subseteq S \subseteq I$, where given a lower prevision $\lpr$ with $dom(\lpr) \subseteq \gambles(\Omega_Z)$ with $T \subseteq Z$, $\lpr_T^{\downarrow{T}}(f) \coloneqq \lpr_T(f^{\uparrow{Z}})= \lpr(f^{\uparrow{Z}})$ for every $f \in \gambles(\Omega_T)$ such that $f^{\uparrow{Z}} \in dom(\lpr)$.
\end{enumerate}

We can show that the map $\underline{h}: \underline{\Psi} \rightarrow \tilde{\underline{\Psi}}$, defined as $(\lpr,S) \rightarrow  (\underline{e}_S(\lpr)_S^{\downarrow{S}}, S)= (\lpr_S^{\downarrow{S}}, S)$, 
is an isomorphism.

First of all notice that, given $(\lpr,S) \in \underline{\Psi}$, $(\lpr_S^{\downarrow{S}},S) \in \tilde{\underline{\Psi}}$. Indeed, $\lpr_S^{\downarrow{S}} = \sigma ( (D^+ \cap \gambles_S)^{\downarrow{S}})$ where $D^+ = \tau(\lpr)$, and $(D^+ \cap \gambles_S)^{\downarrow{S}} \in \Phi(\Omega_S)$ (see Section~\ref{sec:LabInfAlg} ). In fact, for every $f \in \gambles(\Omega_S)$:
\begin{eqnarray*}
\lefteqn{\sigma((D^+ \cap \gambles_S)^{\downarrow{S}})(f) \coloneqq \sup\{ \mu \in \mathbb{R}: f - \mu \in ( D^+ \cap \gambles_S)^{\downarrow{S}} \}} \\&& =
\sup\{ \mu \in \mathbb{R}: f^{\uparrow{I}}- \mu \in D^+ \cap \gambles_S\} \eqqcolon \sigma(D^+)_S ^{\downarrow{S}}(f).
\end{eqnarray*}
So the map $\underline{h}$ is well defined. Now, we can prove the main result.

\begin{theorem} \label{th:LabelIsomPrev}
The map $\underline{h}$ has the following properties.
\begin{enumerate}
\item It maintains combination, null and unit, and projection. Let $(\lpr,S), \; (\lpr_1,S), \; (\lpr_2,T) \in  \underline{\Psi}$:
\begin{eqnarray*}
\underline{h}((\lpr_1,S) \cdot (\lpr_2,T)) &=& \underline{h}(\lpr_1,S) \cdot \underline{h}(\lpr_2,T), \\
 \underline{h}(\sigma({\mathcal{L}}),S) &=& (\sigma({\mathcal{L}}(\Omega_S)),S), \\
 \underline{h}(\sigma({\mathcal{L}^+}),S) &=& (\sigma({\mathcal{L}^+(\Omega_S)}),S), \\
\underline{h}(\underline{\pi}_T(\lpr,S)) &=& \underline{\pi}_T(\underline{h}(\lpr,S)), \textrm{ if } T  \subseteq S,
\end{eqnarray*}
where, $ \sigma({\mathcal{L}}(\Omega_S)) \coloneqq + \infty,$ for all $f \in \gambles(\Omega_S)$.
\item $\underline{h}$ is bijective.
\end{enumerate}
\end{theorem}

\begin{proof}
Let us define the map $\tau':  \underline{\Psi} \mapsto \Psi$ defined as $\tau'(\lpr,S) \mapsto (\tau(\lpr),S)$. It is well defined, because, if $(\lpr,S) \in  \underline{\Psi}$, then $\epsilon_S(\tau(\lpr)) = \tau(\underline{e}_S(\lpr)) = \tau(\lpr)$. Analogously, we can define $\sigma'$. We use the same notation, also for the corresponding maps $\tau': \underline{\tilde{\Psi}}  \mapsto \tilde{\Psi}$ and $\sigma': \tilde{\Psi} \mapsto \underline{\tilde{\Psi}}$ defined respectively as $\tau'(\tilde{\lpr},S) \mapsto (\tau(\tilde{\lpr}),S)$ and $\sigma'(\tilde{D},S) \mapsto (\sigma(\tilde{D}),S)$. Clearly, also in this case they are well defined.

Now, we can proceed with the proof.
\begin{enumerate}
    \item  We have, by definition
    \begin{eqnarray*}
\lefteqn{\underline{h}((\lpr_1,S) \cdot (\lpr_2,T)) \coloneqq \underline{h}  (\lpr_1 \cdot \lpr_2, S \cup T) } \\ && = \underline{h} ( \sigma( D^+_1) \cdot  \sigma(D^+_2), S \cup  T) ,
\end{eqnarray*}
if we call with $D^+_1 \coloneqq \tau(\lpr_1)$ and $D^+_2 \coloneqq \tau(\lpr_2)$.
Now, thanks to Theorem~\ref{th:IsomPrev}, we have
  \begin{eqnarray*}
\lefteqn{\underline{h}((\lpr_1,S) \cdot (\lpr_2,T))  = \underline{h} ( \sigma( D^+_1) \cdot  \sigma(D^+_2), S \cup  T)  } \\ && =\underline{h} ( \sigma( D^+_1 \cdot D^+_2), S \cup  T) \\ && \eqqcolon \underline{h} ( \sigma' (D^+_1 \cdot D^+_2, S \cup T)).
\end{eqnarray*}
As observed before $\underline{h} (\sigma'(D^+, S)) = \sigma'( h(D^+,S))$, for every $(D^+,S) \in \Psi$ such that $D^+ \in C^+ \cup \{\gambles\}$. Indeed, $\underline{h} (\sigma'(D^+, S)) \coloneqq (\sigma(D^+)_S^{\downarrow{S}},S)$ and  $\sigma'( h(D^+,S)) \coloneqq ( \sigma( (D^+ \cap \gambles_S)^{\downarrow{S}}), S)$ and we proved that $\sigma(D^+)_S^{\downarrow{S}} = \sigma( (D^+ \cap \gambles_S)^{\downarrow{S}})$ discussing the well-definiteness  of $\underline{h}$.
Therefore, 
 \begin{eqnarray*}
\lefteqn{\underline{h}((\lpr_1,S) \cdot (\lpr_2,T))  =  \underline{h} ( \sigma' (D^+_1 \cdot D^+_2, S \cup T))  } \\ && =  \sigma' (h (D^+_1 \cdot D^+_2, S \cup T)) \\ && \eqqcolon \sigma' (h ((D^+_1,S) \cdot (D^+_2, T)) ) \\ && = \sigma' (h (D^+_1,S) \cdot h(D^+_2, T)),
\end{eqnarray*}
thanks to Theorem~\ref{th:LabelIsom}.
Now, we claim that 
\begin{eqnarray*}
\sigma' (h (D^+_1,S) \cdot h(D^+_2, T)) = \sigma' (h (D^+_1,S)) \cdot \sigma' (h (D^+_2,T)).
\end{eqnarray*}
Indeed, on the one hand, we have
\begin{eqnarray*}
\lefteqn{\sigma' (h (D^+_1,S) \cdot h(D^+_2, T))} \\ && \coloneqq  \sigma' (((D^+_1 \cap \gambles_S)^{\downarrow{S}}, S) \cdot ((D^+_2 \cap \gambles_T)^{\downarrow{T}}, T))\\ && \coloneqq
( \sigma( \mathcal{C}(((D^+_1 \cap \gambles_S)^{\downarrow{S}})^{\uparrow{S \cup T}} ) \cdot  \mathcal{C}(((D^+_2 \cap \gambles_T)^{\downarrow{T}})^{\uparrow{S \cup T}})), S \cup T)\\ && \coloneqq
( \sigma( \mathcal{C}(D^+_{1,S \cup T}) \cdot  \mathcal{C}(D^+_{2,S \cup T})), S \cup T),
\end{eqnarray*}
where $D^+_{1,S \cup T} \coloneqq ((D^+_1 \cap \gambles_S)^{\downarrow{S}})^{\uparrow{S \cup T}}$ and $D^+_{2,S \cup T} \coloneqq ((D^+_2 \cap \gambles_T)^{\downarrow{T}})^{\uparrow{S \cup T}} $.
On the other hand instead, we have
\begin{eqnarray*}
\lefteqn{ \sigma'(h(D^+_1,S)) \cdot \sigma'(h(D^+_2,T))} \\ && \coloneqq \sigma'( ((D^+_1 \cap \gambles_S)^{\downarrow{S}}, S)) \cdot \sigma'( ((D^+_2 \cap \gambles_T)^{\downarrow{T}}, T)) \\ && \coloneqq (\sigma ((D^+_1 \cap \gambles_S)^{\downarrow{S}}),S) \cdot  (\sigma ((D^+_2 \cap \gambles_T)^{\downarrow{T}}),  T)  \\ && \coloneqq ( \underline{E}^*( (\lpr_{1,S}^{\downarrow{S}})^{\uparrow{S \cup T}} ) \cdot \underline{E}^*( (\lpr_{2,T}^{\downarrow{T}})^{\uparrow{S \cup T}} ), S \cup T).
\end{eqnarray*}
Now, we can show that $(\lpr_{1,S}^{\downarrow{S}})^{\uparrow{S \cup T}} = \sigma(D^+_{1, S \cup T})= \sigma( ((D_1^+ \cap \gambles_S)^{\downarrow{S}})^{\uparrow{S \cup T}}) $. Indeed, 
\begin{eqnarray*}
\lefteqn{\sigma( ((D_1^+ \cap \gambles_S)^{\downarrow{S}})^{\uparrow{S \cup T}})(f) \coloneqq \sup \{ \mu \in \mathbb{R}: f- \mu \in  ((D_1^+ \cap \gambles_S)^{\downarrow{S}})^{\uparrow{S \cup T}}\}=} \\&& = \sup \{ \mu \in \mathbb{R}: f^{\downarrow{S}}- \mu \in  ((D_1^+ \cap \gambles_S)^{\downarrow{S}})\} \eqqcolon (\lpr_{1,S}^{\downarrow{S}})^{\uparrow{S \cup T}}(f),
\end{eqnarray*}
for every $f \in \gambles_S(\Omega_{S \cup T})$ such that $f^{\downarrow{S}} \in dom(\lpr_{1,S}^{\downarrow{S}})$. Analogously, we can show that $(\lpr_{2,T}^{\downarrow{T}})^{\uparrow{S \cup T}} = \sigma(D^+_{2, S \cup T})= \sigma( ((D_2^+ \cap \gambles_T)^{\downarrow{T}})^{\uparrow{S \cup T}}) $. So, we have:
\begin{eqnarray*}
\lefteqn{ ( \underline{E}^*( (\lpr_{1,S}^{\downarrow{S}})^{\uparrow{S \cup T}} ) \cdot \underline{E}^*( (\lpr_{2,T}^{\downarrow{T}})^{\uparrow{S \cup T}} ), S \cup T) } \\ &&
= ( \underline{E}^*( \sigma(D^+_{1,S \cup T} ) ) \cdot \underline{E}^*( \sigma(D^+_{2,S \cup T} ) ), S \cup T)= \\ &&
= ( \sigma( \mathcal{C}( D^+_{1,S \cup T} )) \cdot \sigma( \mathcal{C}(D^+_{2,S \cup T})), S \cup T).
\end{eqnarray*}
In fact, if $D^+_1 = \gambles$, we have the result. Otherwise, $D^+_{1,S \cup T}$ satisfies the hypotheses of Theorem~\ref{th:CommSigmaExt}. Indeed, clearly $0 \notin \edesirs(D^+_{1,S \cup T})$. Moreover, if $f \in D^+_{1,S \cup T} \setminus \gambles^+(\Omega_{S \cup T}) $, then $f \in ((D^+_1 \cap \gambles_S)^{\downarrow{S}})^{\uparrow{S \cup T}} \setminus \gambles^+(\Omega_{S \cup T})$, therefore $f^{\uparrow{I}} \in D_1^+ \cap \gambles_S \setminus \gambles^+$ and then, there exists $\delta >0$ such that $f^{\uparrow{I}} - \delta \in D_1^+ \cap \gambles_S$ that means $f - \delta \in ((D^+_1 \cap \gambles_S)^{\downarrow{S}})^{\uparrow{S \cup T}}$. Therefore, thanks to Theorem~\ref{th:CommSigmaExt}, we have:  $ \underline{E}^*( \sigma( D^+_{1,S \cup T})) = \sigma( \mathcal{C}( D^+_{1,S \cup T}))$. Analogously, we can show that $ \underline{E}^*( \sigma( D^+_{2,S \cup T})) = \sigma( \mathcal{C}( D^+_{2,S \cup T}))$.


Therefore, given the fact that $\mathcal{C}(D^+_{1,S \cup T}),\mathcal{C}(D^+_{1,S \cup T}) \in \Phi^+(\Omega_{S \cup T})$, thanks again to Theorem~\ref{th:IsomPrev}, we have 
\begin{eqnarray*}
\lefteqn{ \sigma'(h(D^+_1,S)) \cdot \sigma'(h(D^+_2,T)) =  ( \sigma( \mathcal{C}( D^+_{1,S \cup T} )) \cdot \sigma( \mathcal{C}(D^+_{2,S \cup T})), S \cup T) =} \\ &=
( \sigma( \mathcal{C}(D^+_{1,S \cup T}) \cdot  \mathcal{C}(D^+_{2,S \cup T})), S \cup T) =\sigma' (h (D^+_1,S) \cdot h(D^+_2, T)).
\end{eqnarray*}
Hence, we have:
 \begin{eqnarray*}
\lefteqn{\underline{h}((\lpr_1,S) \cdot (\lpr_2,T))  = \sigma' (h (D^+_1,S) \cdot h(D^+_2, T))} \\ && = \sigma' (h (D^+_1,S)) \cdot \sigma'(h(D^+_2, T)) \\ && = \underline{h} (\sigma'(D^+_1,S)) \cdot \underline{h}(\sigma'(D^+_2,T))= \\ && = \underline{h}(\lpr_1,S) \cdot \underline{h}(\lpr_2,T).
\end{eqnarray*}
Obviously, $\underline{h}(\sigma(\gambles),S) = (\sigma(\gambles(\Omega_S)),S)$ and $\underline{h}(\sigma(\gambles^+),S) = (\sigma(\gambles^+(\Omega_S)),S)$.

Then, we have
\begin{eqnarray*}¨
\underline{h}(\underline{\pi}_T(\lpr,S)) \coloneqq \underline{h}( \underline{e}_T(\lpr), T) = \underline{h}( \underline{e}_T(\sigma(D^+)), T),
\end{eqnarray*}
if $D^+= \tau(\lpr)$. Therefore,
\begin{eqnarray*}¨
\lefteqn{\underline{h}(\underline{\pi}_T(\lpr,S))   = \underline{h}( \underline{e}_T(\sigma(D^+)), T)  }\\ && 
= \underline{h}( \sigma(\epsilon_T(D^+)), T) \\ &&
\eqqcolon \underline{h}( \sigma'(\pi_T(D^+, S))) \\ &&
= \sigma'( h(\pi_T(D^+, S) ) ) \\ &&
= \sigma'( \pi_T( h(D^+,S)) )  \\ &&
\coloneqq\sigma'( \pi_T ( ( D^+ \cap \gambles_S)^{\downarrow{S}},S) ) \\ &&
\coloneqq \sigma'( (\epsilon_T ( ( D^+ \cap \gambles_S)^{\downarrow{S}}) \cap \gambles_T(\Omega_S))^{\downarrow{T}},T) )\\ &&
\coloneqq (\sigma( (( D^+ \cap \gambles_S)^{\downarrow{S}} \cap \gambles_T(\Omega_S))^{\downarrow{T}}), T),
\end{eqnarray*}
thanks to Theorem~\ref{th:IsomPrev} and Theorem~\ref{th:LabelIsom}.
Now, analogously as before,  
 we have
\begin{eqnarray*}
\lefteqn{\underline{h}(\underline{\pi}_T(\lpr,S)) = (\sigma(( D^+ \cap \gambles_S)^{\downarrow{S}} \cap \gambles_T(\Omega_S))^{\downarrow{T}}), T)} \\ &&
= (\sigma( ( D^+ \cap \gambles_S)^{\downarrow{S}})_T^{\downarrow{T}}, T) \\ &&
= \underline{\pi}_T( \underline{h}(\lpr,S)),
\end{eqnarray*}
thanks to Theorem~\ref{th:IsomPrev} and the fact that $\sigma ( (D^+ \cap \gambles_S)^{\downarrow{S}}) =\sigma ( D^+)_S^{\downarrow{S}} \eqqcolon \lpr_S^{\downarrow{S}}$, so we have the result.

\item Suppose $\underline{h}(\lpr_1,S) =\underline{h}(\lpr_2,T)$. Then we have $S=T$ and $\lpr_{1,S}^{\downarrow{S}} = \lpr_{2,S}^{\downarrow{S}}$, from which we derive that $\lpr_{1,S} = \lpr_{2,S}$ and therefore, $\lpr_1= \underline{E}^*(\lpr_{1,S})= \underline{E}^*(\lpr_{2,S}) = \lpr_{2,S}$.
So the map $\underline{h}$ is injective.

Moreover, for any $(\tilde{\lpr},S) \in \underline{\tilde{\Psi}}$ we have $(\tilde{\lpr},S) = \underline{h}(\underline{E}^*(\tilde{\lpr}^{\uparrow{I}}), S)$, with $(\underline{E}^*(\tilde{\lpr}^{\uparrow{I}}), S) \in \underline{\Psi}$. Clearly, this is true if $\tilde{\lpr} \notin \underline{\mathcal{P}}(\Omega_S)$. Otherwise, let us call $\tilde{D}^+ \coloneqq \tau(\tilde{\lpr})$.  Then, $\tilde{\lpr}^{\uparrow{I}} = \sigma((\tilde{D}^+ )^{\uparrow{I}})$. In fact,
\begin{eqnarray*}
\sigma((\tilde{D}^+) ^{\uparrow{I}})(f) \coloneqq \sup\{ \mu \in \mathbb{R}: f - \mu \in (\tilde{D}^+) ^{\uparrow{I}}\} = \sup\{ \mu \in \mathbb{R}: f^{\downarrow{S}} - \mu \in \tilde{D}^+ \}\eqqcolon \tilde{\lpr}^{\uparrow{I}}(f),
\end{eqnarray*}
for every $f \in \gambles_S(\Omega_I)$, such that $f^{\downarrow{S}} \in dom(\tilde{\lpr})$. Therefore, ${E}^*(\tilde{\lpr}^{\uparrow{I}}) = {E}^*(\sigma((\tilde{D}^+)^{\uparrow{I}}))$.

Hence,
\begin{eqnarray*}
\lefteqn{\underline{e}_S({E}^*(\tilde{\lpr}^{\uparrow{I}})) = \underline{e}_S({E}^*(\sigma((\tilde{D}^+)^{\uparrow{I}})))=  \underline{e}_S(\sigma(\mathcal{C}((\tilde{D}^+)^{\uparrow{I}})))=} \\ &&= \sigma(\epsilon_S(\mathcal{C}((\tilde{D}^+)^{\uparrow{I}})))) =\sigma(\mathcal{C}((\tilde{D}^+)^{\uparrow{I}})) =\underline{e}_S({E}^*(\tilde{\lpr}^{\uparrow{I}})),
\end{eqnarray*}
thanks to Theorem~\ref{th:CommSigmaExt} (that can be applied on $(\tilde{D}^+)^{\uparrow{I}}$), Theorem~\ref{th:IsomPrev} and the fact that $\mathcal{C}((\tilde{D}^+)^{\uparrow{I}})$ has support $S$ (see proof of item 2 of Theorem~\ref{th:LabelIsom}). Therefore, 
$(\underline{E}^*(\tilde{\lpr}^{\uparrow{I}}), S) \in \underline{\tilde{\Psi}}$.
Moreover, 
\begin{eqnarray*}
\lefteqn{\underline{h}(\underline{E}^*(\tilde{\lpr}^{\uparrow{I}}), S) = ({E}^*(\tilde{\lpr}^{\uparrow{I}})_S^{\downarrow{S}},S) = } \\&&= (\sigma(\mathcal{C}((\tilde{D}^+)^{\uparrow{I}}))_S^{\downarrow{S}},S) = (\sigma((\mathcal{C}((\tilde{D}^+)^{\uparrow{I}})  \cap \gambles_S)^{\downarrow{S}}), S) =(\sigma(\tilde{D}^+),S)= (\tilde{\lpr},S).
\end{eqnarray*}
thanks to item 5 of Lemma~\ref{le:cylExtProperties}.
So $h$ is surjective, hence bijective.

\end{enumerate}
\end{proof}

Therefore $\underline{\tilde{\Psi}}$, with the operations defined above, is a labeled information algebra too.

The results of this section show that the coherent lower (and upper) previsions form an information algebra closely related to the information algebra of coherent sets of gambles. This relationship carries over to the labeled versions of the algebras involved. 
However, we have seen that the algebra of coherent sets of gambles is completely atomistic. In the next section we discuss what this means for the algebra of coherent lower previsions.

\section{Linear Previsions} \label{sec:LinPrev}

If, given a coherent lower prevision $\lpr$, it is true that $\lpr(f) = - \lpr(-f)$ for all $f$ in $\mathcal{L}(\Omega)$, that is if a coherent lower prevision $\lpr$ and its respective coherent upper prevision $\overline{P}$ coincide, $\lpr$ is called a linear prevision. Then it is usual to write $\lpr = \upr = P$. Linear previsions have an important role in the theory of imprecise probabilities. Therefore, in this section, they will be examined from the point of view of information algebras. First of all a linear prevision is a coherent lower (and upper) prevision. So, if $\mathcal{P}(\Omega)$ denotes the set of linear previsions on $\mathcal{L}(\Omega)$, we have $\mathcal{P}(\Omega) \subseteq \underline{\mathcal{P}}(\Omega)$. Note that from the third coherence property of lower previsions it follows that $P(f + g) = P(f) + P(g)$, for every $f,g \in \mathcal{L}(\Omega)$. 

Let us concentrate ourselves on the strictly desirable set of gambles associated to a linear prevision $P$,
\begin{eqnarray*}
\tau(P) \coloneqq \{f \in \gambles:P(f) > 0\} \cup \mathcal{L}^+(\Omega) = \{f \in \gambles:-P(-f) > 0\} \cup \mathcal{L}^+(\Omega).
\end{eqnarray*}
We call these sets \emph{maximal strictly desirable} sets of gambles and we indicate them with $M^+$.
It is possible to show that $M^+ \coloneqq \tau(P) = \tau(\sigma(M))$, where $M$ is a maximal set of gambles~\citep{walley91}.
Since maximal sets of gambles are atoms in the algebra of coherent sets of gambles, we may presume that maximal strictly desirable sets of gambles are atoms in the algebra of strictly desirable sets of gambles and linear previsions are also atoms in the algebra of coherent lower previsions. This is indeed the case.

In this section, we will consider the information order on the information algebra of coherent lower previsions defined as usual: given two lower previsions $\lpr_1, \lpr_2 \in \underline{\Phi}(\Omega)$, $\lpr_1 \ge \lpr_2$ if $\lpr_1 \cdot \lpr_2 = \lpr_1$. This order clearly coincides with the usual order on lower previsions restricted to $\underline{\Phi}(\Omega)$. We will consider also the analogous order on elements of $\underline{\tilde{\Psi}}$.


\begin{lemma}\label{lem:linearprevatoms}
Let $\lpr$ be an element of $\underline{\Phi}$ and $P$ a linear prevision. Then $P \leq \lpr$ implies either $\lpr = P$ or $\lpr(f) = \infty$ for all $f \in \mathcal{L}$.
\end{lemma}

\begin{proof}
Clearly $\lpr(f)= + \infty$ for all $f \in \mathcal{L}$ is a possible solution. Consider instead the case in which $\lpr$ is coherent.

From~\cite{walley91}, we know that $\lpr(f) \le \upr(f),$ for all $f \in \mathcal{L}(\Omega)$. Then, we have:
\begin{equation}
    \lpr(f) \le \upr(f) = - \lpr(-f) \le -P(-f)=P(f), \; \forall f \in \mathcal{L}(\Omega). 
\end{equation}
Given the fact that, by hypothesis, we have also $\lpr(f) \ge P(f),$ for all $f \in \mathcal{L}(\Omega)$, we have the result.
\end{proof}

\begin{lemma}
Let $D^+$ be an element of $\Phi^+$ and $M^+$ a maximal strictly desirable set of gambles. Then $M^+ \le D^+$ implies either $D^+=M^+$ or $D^+ = \mathcal{L}(\Omega)$.
\end{lemma}
\begin{proof} Notice that
\begin{equation}
    M^+ \le D^+ \Rightarrow \sigma(M^+) \le \sigma(D^+). 
\end{equation}
Therefore, from Lemma~\ref{lem:linearprevatoms}, we have $\sigma(D^+)= \sigma(M^+)$ or $\sigma(D^+)= \sigma(\gambles)$, from which we derive $D^+= \tau(\sigma(D^+))= \tau (\sigma(M^+))=M^+$ or $D^+= \tau(\sigma(D^+))= \tau(\sigma(\gambles))= \gambles$.
\end{proof}

From Lemma~\ref{lem:linearprevatoms}, we may automatically deduce the properties of atoms in an information algebra, so, for example, we have $\lpr \cdot P = P$ or $\lpr \cdot P = 0$, where here $0$ is the null element $\lpr(f) = \infty$, for all $f \in \mathcal{L}(\Omega)$. The information algebra $\Phi$ of coherent sets of gambles is completely atomistic. It is to be expected that the same holds for the algebra $\underline{\Phi}$ of coherent lower previsions. Let $At(\underline{\Phi}) \coloneqq \mathcal{P}(\Omega)$ be the set of all linear previsions (atoms) and $At(\lpr)$ the set of all linear previsions (atoms) dominating $\lpr \in \Phi$,
\begin{eqnarray*}
At(\lpr) \coloneqq  \{P \in At(\underline{\Phi}):\lpr \leq P\}.
\end{eqnarray*}
Then the following theorem shows that the information algebra $\underline{\Phi}$ is completely atomistic.

\begin{theorem} \label{th:LowPrevAtomistic}
In the information algebra of lower previsions $\underline{\Phi}$, the following holds.
\begin{enumerate}
\item If $\lpr$ is a coherent lower prevision, then $At(\lpr) \neq \emptyset$ and
\begin{eqnarray*}
\lpr = \inf At(\lpr).
\end{eqnarray*}
\item If $A$ is any non-empty subset of linear previsions in $At(\underline{\Phi})$, then 
\begin{eqnarray*}
\lpr \coloneqq \inf A
\end{eqnarray*}
exists and is a coherent lower prevision.
\end{enumerate}
\end{theorem}

For the proof of this theorem, see Theorem 2.6.3 and Theorem 3.3.3 in~\cite{walley91}.

According to this theorem, if $A$ is any non-empty family of linear previsions on $\mathcal{L}(\Omega)$, then $\inf A$ exists and it is a coherent lower prevision $\lpr$. Then we have $A \subseteq At(\lpr)$ and it follows 
\begin{eqnarray*}
\lpr \coloneqq \inf A = \inf At(\lpr).
\end{eqnarray*}

So, the coherent lower prevision $\lpr$ is the lower envelope of the linear previsions (atoms) which dominate it.

As any atomistic information algebra, $\underline{\Phi}$ is embedded in the set algebra of subsets of $At(\underline{\Phi})$ by the map $\lpr \mapsto At(\lpr)$. 
This rises the question of how to characterize the images of $\underline{\Phi}$ in the algebra of subsets of $At(\underline{\Phi})$. The answer is given by the weak*-compactness theorem~\citep[Theorem~3.6.1]{walley91}: the sets $At(\lpr)$ for any coherent lower prevision $\lpr$ are exactly the weak*-compact convex subsets of $At(\underline{\Phi})$ in the weak* topology on $At(\underline{\Phi})$. There are many other sets of linear previsions $A$ whose lower envelope equals $\lpr$.  If $\lpr = \inf A$ and $A \subseteq B \subseteq At(\lpr)$, then $\lpr = \inf B$. In fact there is a minimal set $E \subseteq At(\lpr)$ so that $\lpr = \inf E$ and this is the set of extremal points of the convex set $At(\lpr)$. This follows from the extreme point theorem~\citep{walley91}.

We shall come back to the embedding of the algebra of coherent lower previsions in the algebra of subsets of $At(\underline{\Phi})$ below in the labeled view of the algebra.


We now examine linear previsions in the labeled view of the information algebra of coherent lower previsions, see Section~\ref{sec:LowUpPrev}. 
The elements $(\tilde{P},S)$, where $\tilde{P}$ is a linear prevision on $\gambles(\Omega_S)$, are local atoms in the labeled information algebra $\underline{\tilde{\Psi}}$, that is, if $(\tilde{P},S) \leq (\tilde{\lpr},S)$, then either $ (\tilde{\lpr},S) =(\tilde{P},S)$ or $(\tilde{\lpr},S)= (\sigma(\gambles(\Omega_S)),S)$, that is the null element for label $S$. 
To show it, we can simply use Lemma~\ref{lem:linearprevatoms} on $\Omega_S$. As in Section~\ref{sec:Atoms}, we call $(\tilde{P},S)$ an atom relative to $S$. There follow, as usual, some elementary properties of linear previsions, as atoms, for instance such as the equivalent of Lemma~\ref{le:PropOfAtoms}.

\begin{lemma}
Assume $(\tilde{P},S)$ and $(\tilde{P}_1,S), (\tilde{P}_2,S) $ to be atoms relative to $S$ and $(\tilde{\lpr},S) \in \tilde{\underline{\Psi}}$, with $S \subseteq I$. Then 
\begin{itemize}
\item $(\tilde{P},S) \cdot (\tilde{\lpr},S) = (\sigma(\gambles(\Omega_S)),S)$ or $(\tilde{P},S) \cdot (\tilde{\lpr},S) = (\tilde{P},S)$.
\item If $T \subseteq S$, then $\underline{\pi}_T(\tilde{P},S)$ is an atom relative to $T$.
\item Either $(\tilde{\lpr},S) \leq (\tilde{P},S)$ or $(\tilde{\lpr},S) \cdot (\tilde{P},S) = (\sigma(\gambles(\Omega_S)),S)$,
\item Either $(\tilde{P}_{1},S) \cdot (\tilde{P}_{2},S) = (\sigma(\gambles(\Omega_S)),S)$ or $(\tilde{P}_{1},S) = (\tilde{P}_{2},S)$.
\end{itemize}
\end{lemma}
For a proof of this result on atoms, we refer again to~\cite{kohlas03}.

Just as the labeled information algebra of coherent sets of gambles is atomic, atomistic and completely atomistic, the same holds for the labeled algebra of coherent lower previsions. Let $At(\underline{\tilde{\Psi}}_S)$ be the set of atoms $(\tilde{\pr},S)$ relative to $S$, and $At (\tilde{\lpr},S)$ the set of atoms dominating $(\tilde{\lpr},S) \in \tilde{\underline{\Psi}}_S$.

\begin{itemize}
\item \textit{Atomic:} For any element $(\tilde{\lpr},S) \in \underline{\tilde{\Psi}}_S, S \subseteq I$ with $\tilde{\lpr} \in \underline{\mathcal{P}}(\Omega_S)$, there is an atom relative to $S$,  $(\tilde{P},S)$, so that $(\tilde{\lpr},S) \leq (\tilde{P},S)$. That is, $At(\tilde{\lpr},S)$ is not empty.
\item \textit{Atomistic}: For any element, $(\tilde{\lpr},S) \in \underline{\tilde{\Psi}}_S, S \subseteq I$, with $\tilde{\lpr} \in \underline{\mathcal{P}}(\Omega_S)$,  we have $(\tilde{\lpr},S) = \inf At(\tilde{\lpr},S)$.
\item \textit{Completely Atomistic:} For any, not empty, subset $A$ of $At(\underline{\tilde{\Psi}}_S)$, $\inf A$ exists and belongs to
$\underline{\tilde{\Psi}}_S$, for every $S \subseteq I$.
\end{itemize}

It is further well-known that the local atoms of any atomic labeled information algebra satisfy the following conditions, expressed here for atoms of $\underline{\tilde{\Psi}}$~\citep{kohlas03}.

\begin{lemma} \label{le:TupleSys}
Let
\begin{eqnarray*}
At(\underline{\tilde{\Psi}}) \coloneqq \bigcup_{S \subseteq I} At(\underline{\tilde{\Psi}}_S)
\end{eqnarray*}
and let $(\tilde{P},S),(\tilde{P},T), (\tilde{P}_1,S),(\tilde{P}_{2},T) \in At(\underline{\tilde{\Psi}})$.

Then,
\begin{enumerate}
\item if $T \subseteq d(\tilde{P},S)$, then $d(\underline{\pi}_T( \tilde{P},S)) = T$,
\item if $U \subseteq T \subseteq S$, then $\underline{\pi}_U(\underline{\pi}_T(\tilde{P},S)) = \underline{\pi}_U(\tilde{P},S)$,
\item $\underline{\pi}_S(\tilde{P},S) = (\tilde{P},S)$,
\item if $\underline{\pi}_{S \cap T}(\tilde{P}_{1},S) = \underline{\pi}_{S \cap T}(\tilde{P}_{2},T)$, then there is a $(\tilde{P}, S \cup T)$ $\in At(\underline{\tilde{\Psi}})$ so that $\underline{\pi}_S(\tilde{P}, S \cup T) = (\tilde{P}_{1},S)$ and $\underline{\pi}_T(\tilde{P}, S \cup T) = (\tilde{P}_{2}, T)$,
\item for an element $(\tilde{P},T) \in At(\underline{\tilde{\Psi}})$, if $T \subseteq S$,  there is a $(\tilde{P},S)$ $\in At(\underline{\tilde{\Psi}})$ so that $\underline{\pi}_T(\tilde{P},S) = (\tilde{P},T)$.
\end{enumerate}
\end{lemma}

Most of these properties are immediate consequences from the coherent lower previsions being a labeled information algebra. For a proof of item 4 and 5 we refer to~\cite{kohlas03}. A system like $At(\underline{\tilde{\Psi}})$ with a labeling and a projection operation satisfying the conditions of Lemma~\ref{le:TupleSys} is called a \textit{tuple system}, since it abstracts the properties of concrete tuples as used in relational database systems. And generalized relational algebras can be defined using tuple systems and they turn out to be labeled information algebras~\citep{kohlas03}. In the case of the tuple system $At(\underline{\tilde{\Psi}})$ this goes as follows: a subset $R$ of $At(\underline{\tilde{\Psi}}_S)$ is called a (generalized) relation on $S$. Denote by $\mathcal{R}_S$ all these relations on $S$ and let
\begin{eqnarray*}
\mathcal{R} \coloneqq \bigcup_{S \subseteq I} \mathcal{R}_S.
\end{eqnarray*}
Then, in $\mathcal{R}$, we define the following operations.

\begin{enumerate}
\item \textit{Labeling:} $d(R) \coloneqq S$ if $R \in \mathcal{R}_S$.
\item \textit{Natural Join:} $R_1 \bowtie R_2 \coloneqq \{(\tilde{P},S \cup T) \in At(\underline{\tilde{\Psi}}_{S  \cup T}):\underline{\pi}_S(\tilde{P},S \cup T) \in R_1,\underline{\pi}_T(\tilde{P},S \cup T) \in R_2\}$ if $d(R_1) = S$ and $d(R_2) = T$.
\item \textit{Projection:} $\underline{\pi}_T(R) \coloneqq \{\underline{\pi}_T(\tilde{P},S): (\tilde{P},S) \in R\}$, if $d(R) = S$ and $T \subseteq S$.
\end{enumerate}

With these operations, in particular natural join as combination, the algebra of relations $\mathcal{R}$ becomes a labeled information algebra. Note that the null element of natural join is the empty set and the unit with label $S$ is
$At(\underline{\tilde{\Psi}}_{S})$. This depends only on $At(\underline{\tilde{\Psi}})$ being a tuple system~\citep{kohlas03}. What is more, it turns out that the atomistic labeled algebra of coherent lower previsions is embedded into this generalized relational algebra.

\begin{theorem}
Let $(\tilde{\lpr},S)$ and $(\tilde{\underline{Q}},T)$ be elements of $\underline{\tilde{\Psi}}$. Then
\begin{enumerate}
\item $At(( \tilde{\lpr},S) \cdot (\tilde{\underline{Q}},T)) = At(\tilde{\lpr},S) \bowtie At(\tilde{\underline{Q}},T)$,
\item $At(\underline{\pi}_T(\tilde{\lpr},S)) = \underline{\pi}_T(At(\tilde{\lpr},S))$ if $T \subseteq S$,
\item $At(\sigma(\gambles^+_S),S) = At(\tilde{\underline{\Psi}}_S)$,
\item $At(\sigma(\gambles(\Omega_S)),S)= \emptyset$.
\end{enumerate}
\end{theorem}
Again, this is a theorem of atomistic information algebras and not particular to lower previsions~\citep{kohlas03}. It says that the map $(\tilde{\lpr},S) \mapsto At(\tilde{\lpr},S)$, is an homomorphism between the labeled information algebra of coherent lower previsions and the generalized relational algebra of sets of atoms of $\underline{\tilde{\Psi}}$. Furthermore, since the map is obvioulsy one-to-one, it tells us that $\underline{\tilde{\Psi}}$ is embedded into this relational algebra. This is the labeled version of the embedding of the domain-free algebra into a set algebra of atoms.

\section{The Marginal Problem} \label{sec:Comp}

Consistency, inconsistency, compatibility or incompatibility, whatever is exactly meant by these concepts,  are general properties of information. Here these notions will be defined and studied with respect to coherent sets of gambles and coherent lower previsions, but in the frame and using results of information algebras. Two or more pieces of information can be considered as \emph{consistent}, if their combination is not the null element. 
\begin{definition}[Consistent family of coherent sets of gambles]
A finite family of coherent sets of gambles $D_1,\ldots,D_n$ is \emph{consistent}, or $D_1,\ldots,D_n$ are \emph{consistent}, if $0 \neq D_1 \cdot \ldots \cdot D_n$.
\end{definition}
This is called ``avoiding partial loss" in desirability~\citep{mirzaffalon20}. Otherwise, the family is, or the sets of gambles are, called inconsistent. There is, however, a more restrictive concept of consistency. In the view of information algebras, we prefer to call it \emph{compatibility} since the pieces of information $D_i$ come from, or are part of, the same piece of information $D$~\citep{mirzaffalon20}. 
\begin{definition}[Compatible family of coherent sets of gambles]
A finite family of coherent sets of gambles  $D_1,\ldots,D_n$, where $D_i$ has support $S_i$ for every $i=1,\ldots,n$ respectively, is called \emph{compatible}, or $D_1,\ldots,D_n$ are called \emph{compatible}, if there is a coherent set of gambles $D$ such that $\epsilon_{S_i}(D) = D_i$ for $i=1,\ldots,n$.
\end{definition}
To decide whether a family of $D_i$ is consistent in this sense is also called \emph{the marginal problem}, since extractions are (in the labeled view) projections or marginals. 

In~\cite{mirzaffalon20} is given also a definition of \emph{pairwise compatibility}, 
for coherent sets of gambles.

\begin{definition}[Pairwise compatibility for coherent sets of gambles]
Two coherent sets $D_i$ and $D_j$, where $D_i$ has support  $S_i$ and $D_j$ support $S_j$, are called \emph{pairwise compatible} if 
\begin{eqnarray*}
D_i \cap \mathcal{L}_{S_j} = D_j \cap \mathcal{L}_{S_i}
\end{eqnarray*}
or, equivalently, $\mathcal{E}_{S_j}(D_i) = \mathcal{E}_{S_i}(D_j)$. Analogously, a finite family of coherent sets of gambles $D_i,..., D_n$, where $D_i$ has support $S_i$ for every $i=1,...,n$ respectively, is pairwise compatible, or again $D_i,..., D_n$ are pairwise compatible, if pairs $D_i,D_j$ are pairwise compatible for every $i,j \in \{1, ...n\}$. 
\end{definition}

In terms of the information algebra, this means that
\begin{eqnarray} \label{eq:PairWiseComp}
\epsilon_{S_j}(D_i) = \mathcal{C}(\mathcal{E}_{S_j}(D_i)) = \mathcal{C}(\mathcal{E}_{S_i}(D_j)) = \epsilon_{S_i}(D_j).
\end{eqnarray}
From this it follows that 
\begin{eqnarray*}
\epsilon_{S_i \cap S_j}(D_i) = \epsilon_{S_i \cap S_j}(\epsilon_{S_j}(D_i)) = \epsilon_{S_i \cap S_j}(\epsilon_{S_i}(D_j)) = \epsilon_{S_i \cap S_j}(D_j).
\end{eqnarray*}
In an information algebra in general we could take this as a definition of pairwise compatibility. From this we may recover Eq.~\eqref{eq:PairWiseComp}, since by item 5 of the list of properties of support (Section~\ref{sec:LabInfAlg}), if $S_i$ is a support of $D_i$ and $S_j$ of $D_j$, we have $\epsilon_{S_j}(D_i) = \epsilon_{S_i \cap S_j}(D_i)$ and $\epsilon_{S_i}(D_j) = \epsilon_{S_i \cap S_j}(D_j)$.

Now let $D \coloneqq D_i \cdot D_j$, where $D_i$ and $D_j$ are pairwise compatible, $D_i$ has support $S_i$ and $D_j$ has support $S_j$. Then $\epsilon_{S_i}(D) = D_i \cdot \epsilon_{S_i}(D_j) = D_i \cdot \epsilon_{S_j}(D_i) = D_i$ and also $\epsilon_{S_j}(D) = D_j$. So, pairwise compatible pieces of information are compatible. And, conversely, if $D_i$ and $D_j$ are compatible, then there exists a coherent set $D$, such that $\epsilon_{S_i \cap S_j}(D) = \epsilon_{S_i \cap S_j}(\epsilon_{S_i}(D)) = \epsilon_{S_i \cap S_j}(D_i)$ and similarly $\epsilon_{S_i \cap S_j}(D) = \epsilon_{S_i \cap S_j}(D_j)$. Therefore the two elements are pairwise compatible.

It is well-known that pairwise compatibility among a family of $D_1,\ldots,D_n$ of pieces of information is not sufficient for the family to be compatible. It is also well-known that a sufficient condition to obtain compatibility from pairwise compatibility  is that the family of supports $S_1,\ldots,S_n$ of the $D_i$ satisfy the running intersection property (RIP, that is form a join tree or a hypertree construction sequence).
\begin{description}
\item[RIP] For $i=1$ to $n-1$ there is an index $p(i)$, $i+1 \leq p(i) \leq n$ such that
\begin{eqnarray*}
S_i \cap S_{p(i)} = S_i \cap (\cup_{j=i+1}^n S_j).
\end{eqnarray*}
\end{description}
Then we have the following theorem
~\cite[Theorem~2, Proposition~1]{mirzaffalon20}, a theorem that in fact is a theorem of information algebras in general.

\begin{theorem} \label{th:RIP}
Consider a family of consistent coherent sets of gambles $D_1,\ldots,D_n$ with $n > 1$ where $D_i$ has support $S_i$ for every $i=1, ...,n$ respectively. If  $S_1,\ldots,S_n$ satisfy RIP and $D_1,\ldots,D_n$ are pairwise compatible, then they are compatible and $\epsilon_{S_i}(D_1 \cdot \ldots \cdot D_n) = D_i$ for $i=1,\ldots,n$.
\end{theorem}

\begin{proof}
We give a proof in the framework of general information algebras. Let $Y_i \coloneqq S_{i+1} \cup \ldots \cup S_n$ for $i=1,\ldots,n-1$ and $D \coloneqq D_1 \cdot \ldots \cdot D_n$. Then by RIP
\begin{eqnarray*}
\lefteqn{\epsilon_{Y_1}(D) = \epsilon_{Y_1}(D_1) \cdot D_2 \cdot \ldots \cdot D_n = \epsilon_{Y_1}(\epsilon_{S_1}(D_1)) \cdot D_2 \cdot \ldots \cdot D_n }\\
&&= \epsilon_{S_1 \cap Y_1}(D_1) \cdot D_2 \cdot \ldots \cdot D_n = \epsilon_{S_1 \cap S_{p(1)}}(D_1) \cdot D_2 \cdot \ldots \cdot D_n.
\end{eqnarray*}
But by pairwise compatibility $ \epsilon_{S_1 \cap S_{p(1)}}(D_1) =  \epsilon_{S_1 \cap S_{p(1)}}(D_{p(1)})$, hence by idempotency
\begin{eqnarray*}
\epsilon_{Y_1}(D) = D_2 \cdot \ldots \cdot D_n.
\end{eqnarray*}
By induction on $i$, one shows exactly in the same way that
\begin{eqnarray*}
\epsilon_{Y_i}(D) = D_{i+1} \cdot \ldots \cdot D_n, \; \forall i=1,..,n-1.
\end{eqnarray*}
So, we obtain $\epsilon_{S_n}(D) = \epsilon_{Y_{n-1}}(D) = D_n$.
Now, we claim that $\epsilon_{S_i}(D) = \epsilon_{S_i \cap S_{p(i)}}(D) \cdot D_i$ for every $i=1, ...,n-1$. Since $S_{p(i)} \subseteq Y_i$, we have by RIP
\begin{eqnarray*}
\lefteqn{D_i \cdot \epsilon_{S_i \cap S_{p(i)}}(D) = D_i \cdot \epsilon_{S_i \cap S_{p(i)}}(\epsilon_{Y_i}(D)) = D_i \cdot \epsilon_{S_i \cap Y_i}(\epsilon_{Y_i}(D))= D_i \cdot \epsilon_{S_i}(\epsilon_{Y_i}(D)) } \\
&&= D_i \cdot \epsilon_{S_i}(D_{i+1} \cdot \ldots \cdot D_n) = \epsilon_{S_i}(D_i \cdot D_{i+1} \cdot \ldots \cdot D_n) = \epsilon_{S_i}(\epsilon_{Y_{i-1}}(D))  \\
&&= \epsilon_{S_i}(D).
\end{eqnarray*}
Then, by backward induction, based on the induction assumption $\epsilon_{S_j}(D) = D_j$ for $j > i$, and rooted in $\epsilon_{S_n}(D) = D_n$, for $i = n-1,\ldots 1$, we have by pairwise compatibility
\begin{eqnarray*}
\lefteqn{\epsilon_{S_i}(D) = \epsilon_{S_i \cap S_{p(i)}}(D) \cdot D_i = \epsilon_{S_i \cap S_{p(i)}}(\epsilon_{S_{p(i)}}(D)) \cdot D_i }\\
&&=  \epsilon_{S_i \cap S_{p(i)}}(D_{p(i)}) \cdot D_i = \epsilon_{S_i \cap S_{p(i)}}(D_i) \cdot D_i = D_i.
\end{eqnarray*}
This concludes the proof.
\end{proof} 

Note that the condition $\epsilon_{S_i}(D_1 \cdot \ldots \cdot D_n) = D_i$ for every $i$, implies that the family $D_1,\ldots,D_n$ is compatible. So, this is a sufficient condition for compatibility. This theorem is a theorem of information algebras, it holds not only for coherent sets of gambles, but for any information algebra, in particular for the algebra of coherent lower previsions for instance.

The definition of compatibility and pairwise compatibility depend on the supports (hence indirectly on the $S$-measurability) of the elements $D_i$. But $D_i$ may have different supports. How does this influence compatibility? Assume $D_i$ and $D_j$ are pairwise compatible according to their supports $S_i$ and $S_j$, that is $\epsilon_{S_i \cap S_j}(D_i) = \epsilon_{S_i \cap S_j}(D_j)$. It may be that a set $S'_i \subseteq S_i$ is still a support of $D_i$ and a subset $S'_j \subseteq S_j$ a support of $D_j$. Then
\begin{eqnarray*}
\epsilon_{S'_i \cap S'_j}(D_i) = \epsilon_{S'_i \cap S'_j}(\epsilon_{S_i \cap S_j}(D_i)) = \epsilon_{S'_i \cap S'_j}(\epsilon_{S_i \cap S_j}(D_j)) = \epsilon_{S'_i \cap S'_j}(D_j).
\end{eqnarray*}
So, $D_i$ and $D_j$ are also pairwise compatible relative to the smaller supports $S'_i$ and $S'_j$. The finite supports of a coherent set of gambles $D_i$, have a least support
\begin{eqnarray*}
d_i \coloneqq \bigcap \{S : S \textrm{ support of}\ D_i\}.
\end{eqnarray*}
This is called the \textit{dimension} of $D_i$, it is itself a support of $D_i$ (see item 4 on the list of properties of supports). So, if $D_i$ and $D_j$ are pairwise compatible relative to two of their respective supports $S_i$ and $S_j$, they are pairwise compatible relative to their dimensions $d_i$ and $d_j$. This makes pairwise compatibility independent of an ad hoc selection of supports.

But what about compatibility? Assume that the family $D_1,\ldots,D_n$ is compatible relative to the supports $S_i$ of $D_i$, that is, there exists a coherent set $D$ such that $\epsilon_{S_i}(D) = D_i$. Then we have
\begin{eqnarray*}
\epsilon_{d_i}(D) = \epsilon_{d_i}(\epsilon_{S_i}(D)) = \epsilon_{d_i}(D_i) = D_i.
\end{eqnarray*}
So, the family $D_1,\ldots,D_n$ is also compatible with respect to the system of their dimensions. Again, this makes the definition of compatibility independent from a particular selection of supports. We remark that the set $\{S : S \textrm{ support of}\ D_i\}$ is an upset, that is with any element $S$ in the set an element $S' \supseteq S$ belongs also to the set (item 6 on the list of properties of supports). In fact, 
\begin{eqnarray*}
\{S : S \textrm{ support of}\ D_i\} = \uparrow\!d_i
\end{eqnarray*}
is the set of all supersets of the dimension. Now, assume $D_1,\ldots,D_n$ with $n > 1$, consistent and pairwise compatible. The dimensions $d_i$ may not satisfy RIP, but some sets $S_i \supseteq d_i$ may. 
Then by Theorem~\ref{th:RIP} and this discussion, $D_1,\ldots,D_n$ are compatible. 


From a point of view of information, compatibility of pieces of information $D_1,\ldots,D_n$ is not always desirable. It is a kind of irrelevance or (conditional) independence condition. In fact, if the members of the family $D_1,\ldots,D_n$ with $n > 1$ are consistent, pairwise compatible, and their supports $S_i$ satisfy RIP, then $D_i  = \epsilon_{S_i}(D_1 \cdot \ldots \cdot D_n)$ means that, 
the pieces of information $D_j$ for $j \not= i$ give no new information relative to variables in $S_i$. If, on the other hand, the family of pieces of information $D_1,\ldots,D_n$ is not compatible, but consistent in the sense that $D \coloneqq D_1 \cdot \ldots \cdot D_n \not= 0$, then, if $S_1$ to $S_n$ satisfy RIP, we have that the family $\epsilon_{S_i}(D) \geq D_i$ (in the information order). Indeed $\epsilon_{S_i}(D) \cdot D_i =\epsilon_{S_i}( D_1 \cdot ... \cdot D_{i-1} \cdot D_{i+1} \cdot D_n) \cdot D_i \cdot D_i=\epsilon_{S_i}( D_1 \cdot ... \cdot D_{i-1} \cdot D_{i+1} \cdot D_n) \cdot D_i= \epsilon_{S_i}(D)$. This means that $D_j$ may furnish additional information on the variables in $S_i$ for $i \not= j$. By formula (4.33), p. 119 in \cite{kohlas03}, we have
\begin{eqnarray*}
D = \epsilon_{S_1}(D) \cdot \ldots \cdot \epsilon_{S_n}(D).
\end{eqnarray*}
Obviously the $D'_i \coloneqq \epsilon_{S_i}(D)$ are pairwise compatible and  by definition compatible (this is remark 1 in~\citealp{mirzaffalon20}).

To conclude, note that most of this discussion of compatibility (in particular Theorem~\ref{th:RIP}) depends strongly on idempotency E2 of the information algebra. For instance the valuation algebra corresponding to Bayesian networks is not idempotent, as well as many other semiring-valuation algebras~\citep{kohlaswilson06}. So Theorem~\ref{th:RIP} does not apply in these cases.

We have remarked that compatibility is essentially an issue of information algebras. So, we may expect that concepts and results on compatibility of coherent sets of gambles carry over to coherent lower and upper previsions.

\begin{definition}[Consistent family of coherent lower previsions]
A finite family of coherent lower previsions $\lpr_1,\ldots,\lpr_n$, is \emph{consistent}, or $\lpr_1,\ldots,\lpr_n$ are \emph{consistent}, if $0 \neq \lpr_1 \cdot \ldots \cdot \lpr_n$.
\end{definition}

\begin{definition}[Compatible family of coherent lower previsions]
A finite family of coherent lower previsions $\lpr_1,\ldots,\lpr_n$, where $\lpr_i$ has support $S_i$ for every $i=1,\ldots,n$ respectively, is called \emph{compatible}, or $\lpr_1,\ldots,\lpr_n$ are called \emph{compatible}, if there is a coherent lower prevision $\lpr$ such that $\underline{e}_{S_i}(\lpr) = \lpr_i$ for $i=1,\ldots,n$.
\end{definition}


\begin{definition}[Pairwise compatibility for coherent lower previsions]
Two coherent lower previsions $\lpr_i$ and $\lpr_j$, where $\lpr_i$ has support  $S_i$ and $\lpr_j$ support $S_j$, are called \emph{pairwise compatible}, if 
\begin{eqnarray*}
\underline{e}_{S_ i\cap S_j}(\lpr_j) =  \underline{e}_{S_i \cap S_j}(\lpr_i).
\end{eqnarray*}
Analogously, a finite family of coherent lower previsions $\lpr_i,..., \lpr_n$, where $\lpr_i$ has support $S_i$ for every $i=1, ...,n$ respectively, is pairwise compatible, or again $\lpr_i,..., \lpr_n$ are pairwise compatible, if pairs $\lpr_i,\lpr_j$ are pairwise compatible for every $i,j \in \{1, ...n\}$. 
\end{definition}

Theorem~\ref{th:RIP} carries over, since it is in fact a theorem of information algebras.

\begin{theorem} \label{th:RIP_2}
Consider a family of consistent coherent lower previsions $\lpr_1,\ldots,\lpr_n$ with $n > 1$, where $\lpr_i$ has support $S_i$ for every $i=1, ...,n$ respectively. If $S_1,\ldots,S_n$ satisfy RIP and $\lpr_1,\ldots,\lpr_n$ are pairwise compatible, then they are compatible and $\underline{e}_{S_i}(\lpr_1 \cdot \ldots \cdot \lpr_n) = \lpr_i$ for $i=1,\ldots,n$.
\end{theorem}

Of course, there are close relations between compatibility of coherent sets of gambles and coherent lower previsions by the homomorphism between the related algebras. If a family of coherent sets $D_1,\ldots,D_n$ with supports $S_1,\ldots,S_n$ respectively, is compatible, then the associated family of coherent lower previsions $\sigma(D_1), \ldots, \sigma(D_n)$ is compatible too, since $\underline{e}_{S_i}(\sigma(D)) = \sigma(\epsilon_{S_i}(D)) = \sigma(D_i)$. Conversely, if $\lpr_1,\ldots,\lpr_n$ is a compatible family of coherent lower previsions with support $S_1, \ldots,S_n$ respectively, then there is a compatible family of strictly desirable sets of gambles $D_i^+ \coloneqq \tau(\lpr_i)$. 

\section{Outlook}

This paper presents a first approach to information algebras related to coherent sets of desirable gambles and coherent lower and upper previsions.
This leads us the possibility to abstract away properties of desirability that can be regarded as properties of the more general algebraic structure of information algebras rather than special ones of desirability.

\cite{decooman2005} however, pursued a similar purpose. He showed indeed that there is a common order-theoretic structure that he calls \emph{belief structure}, underlying many of the
models for representing beliefs in the literature such as, for example, classical propositional logic, almost desirable sets of gambles or lower and upper previsions.

 There are surely important and interesting connections between De Cooman’s belief structures~\citep{decooman2005} and information algebras. In particular between belief structures and information algebras based on closure operations, linked with information systems~\citep{kohlas03}. This hints at a profound connection between the two theories, which certainly deserves careful study. However, this is a subject which has yet to be worked out and may advance both information algebra theory as well as belief structures.

There are also other aspects and issues which are not addressed here. In particular, we limit our work to multivariate models, where coherent sets of gambles and  coherent lower previsions 
represent pieces of information or belief relative to sets $S$ of variables. 
However, more general possibility spaces can be considered. In the view of information algebras, this translates in considering coherent sets of gambles and coherent previsions as pieces of information regarding more general partitions of the set of possibilities. This case has been analyzed in more detail in~\cite{kohlas21}.
Moreover, another important issue which is not been addressed here is the issue of conditioning. It should be analyzed both for the multivariate that for the more general cases of possibility spaces.
This would also serve to analyze the issue of conditional independence, which seems to be fundamental for any theory of information.

\bibliography{text}

\begin{thebibliography}{}

\bibitem[\protect\citename{Davey \& Priestley, }2002]{daveypriestley97}
{\sc Davey, B.~A., \& Priestley, H.~A.} 2002.
\newblock {\em Introduction to Lattices and Order}.
\newblock Cambridge University Press.

\bibitem[\protect\citename{De~Cooman, }2005]{decooman2005}
{\sc De~Cooman, G.} 2005.
\newblock Belief models: an order-theoretic investigation.
\newblock {\em Annals of Mathematics and Artificial Intelligence}, {\bf 45}(1),
  5--34.

\bibitem[\protect\citename{De~Cooman \& Quaeghebeur, }2012]{CooQua12}
{\sc De~Cooman, G., \& Quaeghebeur, E.} 2012.
\newblock Exchangeability and sets of desirable gambles.
\newblock {\em International Journal of Approximate Reasoning}, {\bf 53},
  563--305.

\bibitem[\protect\citename{De~Cooman {\em et~al.}, }2011]{decooman2011}
{\sc De~Cooman, G., Miranda, E., \& Zaffalon, M.} 2011.
\newblock Independent natural extension.
\newblock {\em Artificial Intelligence}, {\bf 175}(12-13), 1911--1950.

\bibitem[\protect\citename{Kohlas, }2003]{kohlas03}
{\sc Kohlas, J.} 2003.
\newblock {\em Information Algebras: Generic Structures for Inference}.
\newblock Springer-Verlag.

\bibitem[\protect\citename{Kohlas, }2017]{kohlas17}
{\sc Kohlas, J.} 2017.
\newblock {\em Algebras of Information. A New and Extended Axiomatic Foundation
  {@ONLINE}}.

\bibitem[\protect\citename{Kohlas \& Schmid, }2020]{kohlasschmid16}
{\sc Kohlas, J., \& Schmid, J.} 2020.
\newblock {\em Commutative Information Algebras and their Representation
  Theory}.

\bibitem[\protect\citename{Kohlas \& Wilson, }2006]{kohlaswilson06}
{\sc Kohlas, J., \& Wilson, N.} 2006.
\newblock {\em Exact and Approximate Local Computation in Semiring Induced
  Valuation Algebras}.
\newblock Tech. rept. 06-06. Department of Informatics, University of Fribourg.

\bibitem[\protect\citename{Kohlas {\em et~al.}, }2021]{kohlas21}
{\sc Kohlas, J., Casanova, A., \& Zaffalon, M.} 2021.
\newblock {\em Information algebras of coherent sets of gambles in general
  possibility spaces}.
\newblock Accepted for publication in {\em Proceedings of Machine Learning
  Research}.

\bibitem[\protect\citename{Lauritzen \& Spiegelhalter,
  }1988]{lauritzenspiegelhalter88}
{\sc Lauritzen, S.~L., \& Spiegelhalter, D.~J.} 1988.
\newblock Local computations with probabilities on graphical structures and
  their application to expert systems.
\newblock {\em J. Royal Statis. Soc. B}, {\bf 50}, 157--224.

\bibitem[\protect\citename{Miranda \& Zaffalon, }2020]{mirzaffalon20}
{\sc Miranda, E., \& Zaffalon, M.} 2020.
\newblock {\em Compatibility, Desirability, and the Running Intersection
  Property}.

\bibitem[\protect\citename{Shafer, }1991]{shafer91}
{\sc Shafer, G.} 1991.
\newblock {\em An Axiomatic Study of Computation in Hypertrees}.
\newblock Working Paper 232. School of Business, University of Kansas.

\bibitem[\protect\citename{Shenoy \& Shafer, }1990]{shenoyshafer90}
{\sc Shenoy, P.~P., \& Shafer, G.} 1990.
\newblock Axioms for probability and belief-function proagation.
\newblock {\em Pages  169--198 of:} {\sc Shachter, Ross~D., Levitt, Tod~S.,
  Kanal, Laveen~N., \& Lemmer, John~F.} (eds), {\em Uncertainty in Artificial
  Intelligence 4}.
\newblock Machine intelligence and pattern recognition, vol. 9.
\newblock Amsterdam: Elsevier.

\bibitem[\protect\citename{Troffaes \& De~Cooman, }2014]{troffaes2014}
{\sc Troffaes, M.~CM, \& De~Cooman, G.} 2014.
\newblock {\em Lower Previsions}.
\newblock John Wiley \& Sons.

\bibitem[\protect\citename{Walley, }1991]{walley91}
{\sc Walley, P.} 1991.
\newblock {\em Statistical Reasoning with Imprecise Probabilties}.
\newblock Chapman and Hall.

\end{thebibliography}
\bibliographystyle{authordate3}


\end{document}